\documentclass{article}

\usepackage[final,nonatbib]{neurips_2020}

\usepackage{microtype}
\usepackage{graphicx}
\usepackage{xcolor}
\usepackage{subfigure}
\usepackage{caption}
\usepackage{booktabs} 
\usepackage{amsfonts}
\usepackage{amsmath}
\usepackage{bbm}
\usepackage{amsthm}
\usepackage{multirow}
\usepackage{bm}
\usepackage{soul} 
\usepackage[normalem]{ulem} 
\setstcolor{blue}
\usepackage{hyperref}
\usepackage{wrapfig}
\usepackage{caption}
\usepackage{tabularx}
\usepackage{algorithm}
\usepackage[noend]{algpseudocode}

\title{CO-Optimal Transport}
\author{Ievgen Redko\thanks{Authors contributed equally.} \\
  Univ Lyon, UJM-Saint-Etienne, CNRS, UMR 5516\\
 F-42023, Saint-Etienne \\
  \texttt{ievgen.redko@univ-st-etienne.fr} \And
  Titouan Vayer$^\ast$ \\
  Univ. Bretagne-Sud, CNRS, IRISA \\
  F-56000 Vannes\\
  \texttt{titouan.vayer@irisa.fr} \\
  \And
  R\'emi Flamary$^\ast$ \\
  \'Ecole Polytechnique, CMAP, UMR 7641 \\
  F-91120 Palaiseau \\
  \texttt{remi.flamary@polytechnique.edu} 
  \And Nicolas Courty$^\ast$ \\
  Univ. Bretagne-Sud, CNRS, IRISA \\
  F-56000 Vannes \\
  \texttt{nicolas.courty@irisa.fr} \\
}

\begin{document}

\maketitle

\newcommand{\nc}[1]{{\color{purple}#1}}
\newcommand{\tv}[1]{{\color{blue}#1}}
\newcommand{\rf}[1]{{\color{teal}#1}}
\newcommand{\ir}[1]{{\color{orange}#1}}
\newcommand{\iev}[2]{\textcolor{orange}{#2}}

\def\G{\pi}
\def\GG{\boldsymbol\G}
\def\Gs{\pi^s}
\def\GGs{\boldsymbol\pi^s}
\def\Gv{\pi^v}
\def\GGv{\boldsymbol\pi^v}
\def\w{{\bf w}}
\def\v{{\bf v}}
\def\x{{\bf x}}
\def\y{{\bf y}}
\def\X{{\bf X}}
\def\Y{{\bf Y}}
\def\L{{\bf L}}
\def\R{{\mathbb{R}}}
\def\vec{{\text{vec}}}
\def\tr{{\text{tr}}}
\def\one{{\mathbf{1}}}

\newcommand{\CCOT}{\texttt{CCOT}}
\newcommand{\CCOTGW}{\texttt{CCOT-GW}}
\newcommand{\MovieL}{\textsc{MovieLens}}

\newcommand{\COOT}{\text{COOT}}

\newcommand{\bz}{\mathbf{z}}
\newcommand{\bw}{\mathbf{w}}

\newcommand{\xbf}{\mathbf{x}}
\newcommand{\ybf}{\mathbf{y}}
\newcommand{\zbf}{\mathbf{z}}

\newcommand{\ie}{\textit{i.e.}}
\newcommand\abs[1]{\left\lvert#1\right\rvert}

\newtheorem{prop}{Proposition}
\newtheorem{theo}{Theorem}
\newtheorem{lemma}{Lemma}

\graphicspath{ {./imgs/} }




\begin{abstract}

Optimal transport (OT) is a powerful geometric and probabilistic tool for finding correspondences and measuring similarity between two distributions. Yet, its original formulation relies on the existence of a cost function between the samples of the two distributions, which makes it impractical when they are supported on different spaces. To circumvent this limitation, we propose a novel OT problem, named COOT for CO-Optimal Transport, that simultaneously optimizes two transport maps between both samples and features, contrary to other approaches that either discard the individual features by focusing on pairwise distances between samples or need to model explicitly the relations between them. We provide a thorough theoretical analysis of our problem, establish its rich connections with other OT-based distances and demonstrate its versatility with two machine learning applications in heterogeneous domain adaptation and co-clustering/data summarization, where COOT leads to performance improvements over the state-of-the-art methods. 

\end{abstract}

\allowdisplaybreaks

\section{Introduction}
\label{sec:intro}
The problem of comparing two sets of samples arises in many fields in machine learning, such as manifold alignment \cite{Cui:2014}, image registration \cite{Haker:2001},
unsupervised word and sentence translation \cite{rapp-1995-identifying} among others. When correspondences between the sets are known \textit{a priori}, one can align them with a global transformation of the features, \textit{e.g}, with the widely used \emph{Procrustes
analysis} \cite{oro2736,Goodall:1991}. For unknown correspondences, other popular alternatives to this method include correspondence free manifold alignment procedure \cite{Wang2009ManifoldAW}, soft assignment coupled with a Procrustes matching \cite{Rangarajan:1997} or Iterative closest point and its variants for 3D shapes \cite{Besl_ICP:1992,yang2020teaser}. 

{When one models} the considered sets of samples as empirical probability distributions{,} Optimal Transport (OT) framework {provides a solution} to find, without supervision, a soft-correspondence map between them given by an \emph{optimal coupling}. 
{OT-based approaches} have been used with success in numerous applications such as embeddings'
alignments \cite{alavarez:2019,Grave2018UnsupervisedAO} and Domain Adaptation (DA)
\cite{DBLP:journals/pami/CourtyFTR17} to name a few. {However, one important limit of using OT} for such tasks is that the two sets are assumed to lie in the same space so that the cost between samples across them can be computed. This
major drawback {does not allow OT to} handle correspondence estimation across heterogeneous
spaces, preventing its application in problems such as, for instance, heterogeneous DA (HDA). To circumvent this restriction, one
may rely on the Gromov-Wasserstein distance (GW)
\cite{memoli_gw}{: a} non-convex quadratic OT problem {that} finds the correspondences between two sets of samples based on their pairwise intra-domain
similarity (or distance) matrices. Such an approach was successfully applied {to} sets of samples {that} do not lie in the same
Euclidean space, \textit{e.g} for shapes \cite{solomon_entropic_2016}, word embeddings
\cite{alvarez-melis_gromov-wasserstein_2018} {and HDA \cite{ijcai2018-412} mentioned previously}. One important limit of GW is that it finds the samples' correspondences but discards the relations between the features by considering pairwise similarities only.


In this work, we propose a novel OT approach called CO-Optimal transport (\COOT) that
simultaneously infers the correspondences between the samples \emph{and} the features of two arbitrary
sets. 
{Our new formulation includes GW as a special case, and has an extra-advantage of working with raw data directly without needing to compute, store and choose computationally demanding similarity measures required for the latter.} Moreover, \COOT\ provides a meaningful mapping between both instances and
features across the two datasets thus having the virtue of being interpretable. We thoroughly analyze the proposed problem, derive an optimization procedure for it and
highlight several insightful links to other approaches. On the practical side,
we provide evidence of its versatility in machine learning by putting forward
two applications in HDA and co-clustering where our approach achieves state-of-the-art results. 

The rest of this paper is organized as follows. We introduce the \COOT\ problem
in Section 2 and give an optimization routine for solving it efficiently. In
Section 3, we show how \COOT\ is related to other OT-based distances and recover efficient solvers for some of them in particular cases. 
Finally, in Section 4, we present an experimental study providing highly
competitive results in {HDA} and co-clustering
compared to several baselines.

\section{CO-Optimal transport (\COOT)}
\label{sec:method}
\paragraph{Notations.} The simplex histogram with $n$ bins is denoted by $\Delta_{n}=\{\w \in (\mathbb{R}_{+})^{n}:\ \sum_{i=1}^{n} w_{i}=1\}$. We further denote by $\otimes$ the tensor-matrix multiplication, \ie, for a tensor $\L=(L_{i,j,k,l})$, $\L\otimes \mathbf{B}$ is the matrix $(\sum_{k,l} L_{i,j,k,l}B_{k,l})_{i,j}$. We use $\langle \cdot, \cdot \rangle$ for the matrix scalar product associated with the Frobenius norm $\|\cdot\|_{F}$ and $\otimes_{K}$ for the Kronecker product of matrices, \ie, $\mathbf{A} \otimes_{K} \mathbf{B}$ gives a tensor $\L$ such that $L_{i,j,k,l}=A_{i,j} B_{k,l}$. We note $\mathbb{S}_{n}$ the group of permutations of $\{1,\cdots,n\}=[\![n]\!]$. Finally, {we write $\bm{1}_d \in \mathbb{R}^d$ for a $d$-dimensional vector of ones} and denote all matrices by upper-case bold letters (\ie, $\X$) or lower-case Greek letters (\ie, $\GG$); all vectors are written in lower-case bold (\ie, $\xbf$).

\subsection{CO-Optimal transport optimization problem}

We consider two datasets represented by matrices $\X=[\x_1,\dots,\x_n]^T\in\mathbb{R}^{n\times d}$ and $\X'=[\xbf'_1,\dots,\xbf'_{n'}]^T\in\mathbb{R}^{n'\times d'}${, where in general we assume that $n\neq n'$ and $d\neq d'$.} In what follows, the rows of the datasets are denoted as \emph{samples} and its columns as \emph{features}. {We endow the samples $(\xbf_i)_{i \in [\![n]\!]}$ and $(\xbf'_i)_{i \in [\![n']\!]}$ with weights $\w=[w_1,\dots,w_n]^\top \in\Delta_n$ and $\w'=[w_1',\dots,w_{n'}']^\top \in\Delta_{n'}$ that both lie in the simplex so as to define empirical distributions supported on $(\xbf_i)_{i \in [\![n]\!]}$ and $(\xbf'_i)_{i \in [\![n']\!]}$. In addition to these distributions, we similarly associate weights given by vectors $\v\in\Delta_d$ and $\v'\in\Delta_{d'}$ with features. Note that
when no additional information is available about the data, all the weights' vectors
can be set as uniform.} 


We define the CO-Optimal Transport problem as follows:
\begin{equation}
 \label{eq:co-optimal-transport}
\min_{\begin{matrix}\GGs \in\Pi(\w,\w') \\ \GGv\in\Pi(\v,\v')\end{matrix}} \quad \sum_{i,j,k,l} L(X_{i,k},X'_{j,l})\GGs_{i,j}\GGv_{k,l} =\min_{\begin{matrix}\GGs \in\Pi(\w,\w') \\ \GGv\in\Pi(\v,\v')\end{matrix}} \langle \L(\X,\X') \otimes \GGs, \GGv \rangle\\
\end{equation}
where $L:\mathbb{R}\times \mathbb{R} \rightarrow \mathbb{R}_+$ is a divergence measure between 1D variables, $\L(\X,\X')$ is the $d\times d'\times n\times n'$ tensor of all
pairwise divergences between the elements of $\X$ and $\X'$, and
$\Pi(\cdot,\cdot)$ is the set of linear transport constraints defined for $\w,\w')$ as $\Pi(\w,\w') =\{\GG| \GG\geq \mathbf{0}, \GG \mathbf{1}_{n'}=\w, \GG^\top \mathbf{1}_{n}=\w'\}$ and similarly for $\v,\v'$. Note that problem \eqref{eq:co-optimal-transport} seeks for a simultaneous transport $\GGs$ between samples and a transport
$\GGv$ between features across distributions. In the following, we write $\COOT(\X,\X',\w,\w',\v,\v')$ (or $\COOT(\X,\X')$ when it is
clear from the context) to denote the objective value of the optimization problem \eqref{eq:co-optimal-transport}. 

Equation \eqref{eq:co-optimal-transport} can be also extended to the entropic regularized case favoured in the OT community for remedying the heavy computation burden of OT and reducing its sample complexity \cite{cuturi:2013,Altschuler:2017,genevay:2019}. This leads to the following problem:
\begin{equation}
 \label{eq:co-optimal-transport-reg}
    \min_{\GGs \in\Pi(\w,\w'), \GGv\in\Pi(\v,\v')} \langle \L(\X,\X') \otimes \GGs, \GGv \rangle
    + \Omega(\GGs,\GGv)
\end{equation}
where for $\epsilon_{1},\epsilon_{2}>0$, the regularization term writes as $\Omega(\GGs,\GGv)=  \epsilon_{1} H(\GGs|\w
\w'^{T})+\epsilon_{2}H(\GGv|\v\v'^{T})$ with
$H(\GGs|\w\w'^{T})=\sum_{i,j} \log(\frac{\pi^s_{i,j}}{w_{i}w'_{j}})\pi^s_{i,j}$
being the relative entropy. Note that similarly to OT \cite{cuturi:2013} and GW
\cite{peyre2016gromov}, adding the regularization term can lead to a more robust
estimation of the transport matrices but prevents them from being sparse.

\begin{figure*}[!t]
    \centering
    \includegraphics[width=\linewidth]{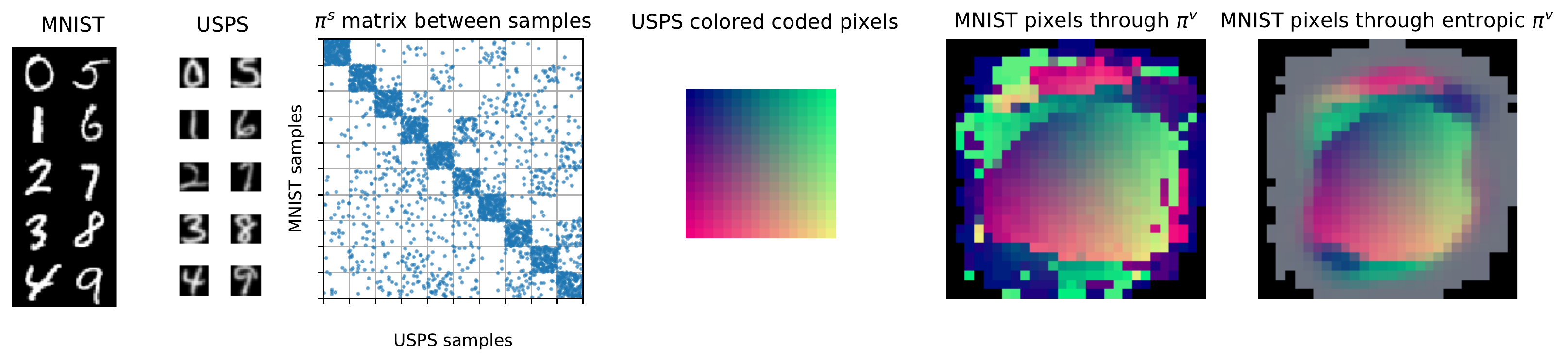}
    \caption{Illustration of \COOT\ between MNIST and USPS datasets. \textbf{(left)} samples from MNIST and USPS data sets; \textbf{(center left)} Transport matrix $\GGs$ between samples sorted by class; \textbf{(center)} USPS image with pixel{s} colored \emph{w.r.t.} their 2D position; \textbf{(center right)} transported colors on MNIST image using $\GGv$, black pixels correspond to non-informative MNIST pixels always at 0; \textbf{(right)} transported colors on MNIST image using $\GGv$ with entropic regularization.}
    \label{fig:mnist_usps}
\end{figure*}

\paragraph{Illustration of \COOT} In order to illustrate {our proposed} \COOT\
{method} and to explain the intuition behind it, we solve the optimization problem \eqref{eq:co-optimal-transport} {using the algorithm described in section \ref{sec:optim}}
between two classical digit recognition datasets: MNIST and USPS. We choose these particular datasets for our illustration as they contain images of different
resolutions (USPS is 16$\times$16 and MNIST is 28$\times$28) that belong to the
same classes (digits between 0 and 9). Additionally, the digits are also
slightly differently centered as illustrated on the examples in the left part of
Figure \ref{fig:mnist_usps}. {Altogether,} this means that without specific
pre-processing, the images do not lie in the same {topological} space and thus
cannot be compared directly using conventional distances. We randomly select 300
images per class in each dataset, normalize magnitudes of pixels to $[0,1]$ and consider digit images as \emph{samples} while each pixel acts as a \emph{feature} leading to 256 and 784 features for USPS and MNIST respectively. We use
uniform weights for $\w,\w'$ and normalize average values of each pixel for $\v,\v'$ in order to discard non-informative ones that are always equal to 0.

The result of solving problem \eqref{eq:co-optimal-transport} is reported in
Figure \ref{fig:mnist_usps}. In the center-left part, we provide the coupling
$\GGs$ between the samples, {\textit{i.e} the different images}, sorted by class and observe that 67\% of mappings occur between the samples from the same class as indicated by block diagonal structure of the coupling matrix.
The coupling $\GGv$, in its turn, describes the relations between {the features, \textit{i.e} the pixels,} in both domains. 
To visualize it, we color-code the pixels of the source USPS image and use $\GGv$ to transport the colors on a target MNIST image so that its pixels are defined as  convex combinations of colors from the former with coefficients given by $\GGv$. 
The corresponding results are shown in the right part of
Figure~\ref{fig:mnist_usps} for both the original \COOT\ and its entropic
regularized counterpart. From these two images, we can observe that colored
pixels appear only in the central areas and exhibit a strong spatial coherency despite the fact that 
{the geometric structure of the image is totally unknown to the optimization problem, as each pixel is treated as an independent variable.}
 \COOT\ has recovered a meaningful spatial transformation between the 
two datasets in a completely unsupervised way, {different from trivial rescaling of images that one may expect when aligning USPS digits occupying the full image space and MNIST digits
liying in the middle of it ({for further evidence, other visualizations are given in the
\hyperlink{appendix}{supplementary material}}).

\paragraph{\COOT\ as a billinear program} \COOT\ is an indefinite Bilinear Program (BP)
problem \cite{gallo:1977}: a special case of a Quadratic Program (QP) with linear constraints 
for which  there exists an optimal solution lying on extremal points of the polytopes
$\Pi(\w,\w')$ and $\Pi(\v,\v')$ \cite{pardalos:1987,horst1996global}. When $n=n',d=d'$ and weights $\w=\w'=\frac{\bm{1}_n}{n}, \v=\v'=\frac{\bm{1}_d}{d}$ are uniform, Birkhoff’s theorem
\cite{birkhoff:1946} states that the set of extremal points of
$\Pi(\frac{\one_{n}}{n},\frac{\one_{n}}{n})$ and
$\Pi(\frac{\one_{d}}{d},\frac{\one_{d}}{d})$ are the set of permutation matrices
so that there exists an optimal solution $(\GGs_{*},\GGv_{*})$ 
which transport maps are supported on two permutations 
$\sigma_{*}^{s},\sigma_{*}^{v} \in \mathbb{S}_{n} \times \mathbb{S}_{d}$. 

The BP problem is also related to the Bilinear Assignment Problem (BAP)
where $\GGs$ and $\GGv$ are searched in the set of permutation matrices. The latter was shown to be NP-hard if $d=O(\sqrt[\leftroot{-2}\uproot{2}r]{n})$ for fixed $r$ and solvable in polynomial time if $d=O(\sqrt{\log(n)})$
\cite{Custic:2016}. In this case, we look for the best permutations of the rows
and columns of our datasets that lead to the smallest cost. \COOT\ provides a
tight convex relaxation of the BAP by 1) relaxing the constraint 
set of permutations into the convex set of doubly stochastic matrices and 2) ensuring that two problems are equivalent, \ie, one can always find a pair of permutations that minimizes \eqref{eq:co-optimal-transport}, as explained in the paragraph above.



\label{sec:theoretical}

\subsection{{Properties of \COOT}}
\label{sec:}
Finding a meaningful similarity measure between datasets is useful in many machine learning tasks as pointed out, \textit{e.g} in \cite{alvarezmelis2020geometric}. To this end, \COOT\ induces a distance between datasets $\X$ and $\X'$ and it vanishes $\textit{iff}$ they are the same up to a permutation of rows and columns as established below\footnote{All proofs and theoretical results of this paper are detailed in the \hyperlink{appendix} {supplementary materials}.}.
\label{sec:metric_properties}
\begin{prop}[\COOT\ is a distance]
Suppose $L=|\cdot|^{p}, p \geq 1$, $n=n',d=d'$ and that the weights $\w,\w',\v,\v'$ are uniform. Then $\COOT(\X,\X')=0$ \textit{iff} there exists a permutation of the samples $\sigma_{1} \in \mathbb{S}_{n}$ and of the features $\sigma_{2} \in \mathbb{S}_{d}$, \textit{s.t}, $\forall i,k \ \X_{i,k}=\X'_{\sigma_{1}(i),\sigma_{2}(k)}$. Moreover, it is symmetric and satisfies the triangular inequality as long as $L$ satisfies the triangle inequality, \ie, $\COOT(\X,\X'')\leq \COOT(\X,\X')+\COOT(\X',\X'').$
\end{prop}
{Note that in the general case when $n\neq n', d\neq d'$, positivity and triangle inequality still hold but $\COOT(\X,\X')>0$.}
Interestingly, our result generalizes the metric property proved in \cite{faliszewski19} for the election isomophism problem with this latter result being {valid} only for the BAP case (for a discussion on the connection between \COOT\ and the work of \cite{faliszewski19}, see \hyperlink{appendix}{supplementary materials}). Finally, we note that this metric property means that \COOT\ can be used as a divergence in a large
number of potential applications\iev{where GW was proved to be efficient}{} as, for instance, in generative learning
\cite{bunne2019learning}.



\label{sec:optim}

\subsection{Optimization algorithm and complexity \label{sec:optimsec}}

        \begin{algorithm}[t]
        \caption{\label{alg:bcd}
       BCD for \COOT}
        \begin{algorithmic}[1]
        \State $\pi^{s}_{(0)}\leftarrow \w\w'^{T},\pi^{v}_{(0)}\leftarrow \v\v'^{T}, k \leftarrow 0$
          \While {$k < $ maxIt {\bf and} $err >$ 0} 
          \State $\GGv_{(k)} \leftarrow OT(\v,\v',\L(\X,\X')\otimes \GGs_{(k-1)})$ // \text{ OT problem on the samples}
          \State $\GGs_{(k)} \leftarrow OT(\w,\w',\L(\X,\X')\otimes \GGv_{(k-1)})$ // \text{ OT problem on the features}
          \State $err \leftarrow ||\GGv_{(k-1)} - \GGv_{(k)}||_F$
          \State $k\leftarrow k+1$          
          \EndWhile
        \end{algorithmic}
        \end{algorithm}


\begin{wrapfigure}[14]{r}{0.33\linewidth}
    \vspace{-8mm}
\includegraphics[width=1\linewidth]{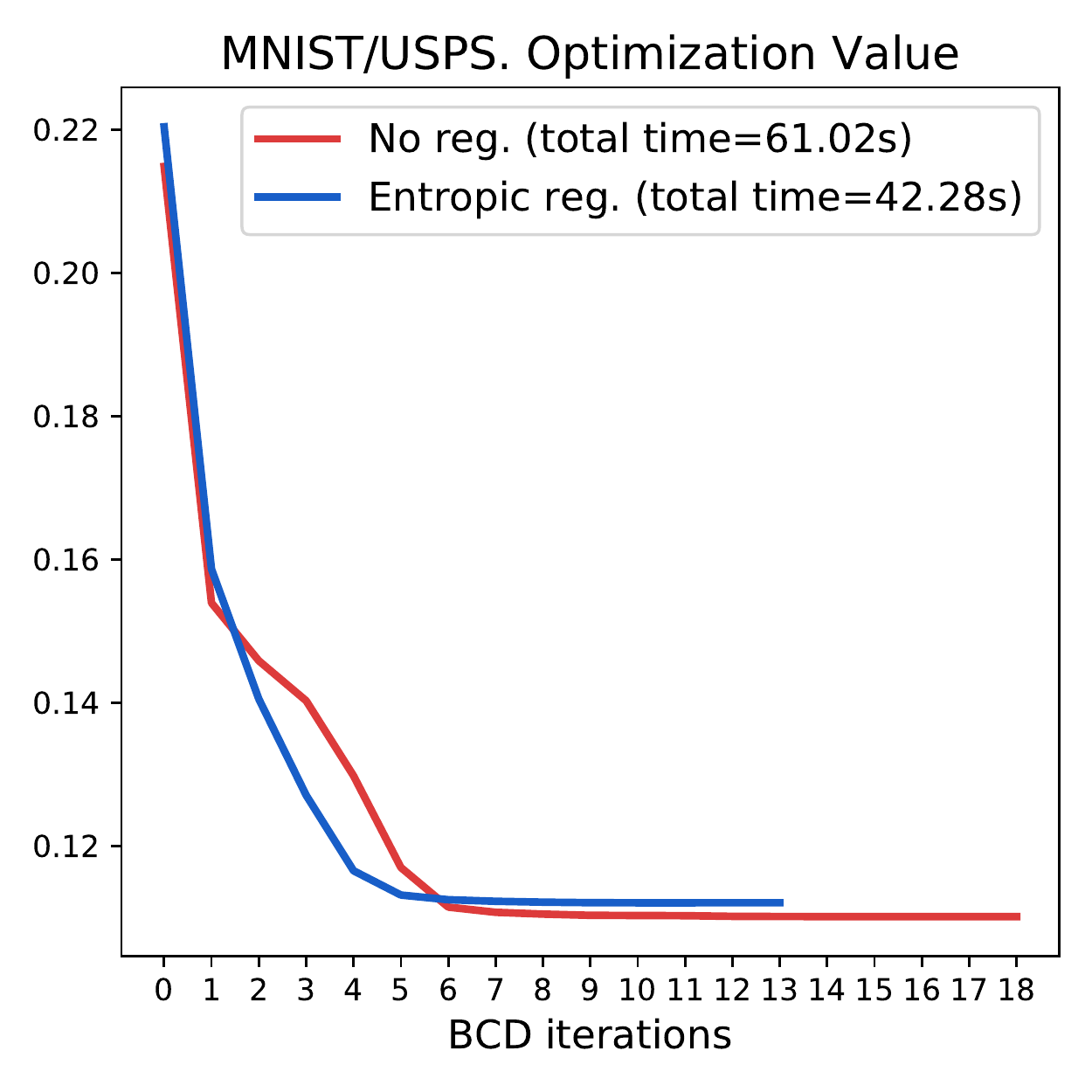}
\caption{{\small \COOT\ loss during the BCD for the MNIST/USPS task.}}
\label{fig:loss_bcd}
\end{wrapfigure}

Even though solving  \COOT\ exactly may be NP-hard,
in practice computing a solution can be done rather efficiently. To this end, we propose to use
Block Coordinate Descent (BCD) that consists in iteratively solving the problem
for $\GGs$ or $\GGv$ with the other kept fixed. 
Interestingly, this boils down to solving at each step a classical OT problem that requires $O(n^3\log(n))$ operations with a network simplex algorithm. The pseudo-code of the proposed algorithm, known as the ``mountain climbing procedure"
\cite{konno:1976},  is given in Algorithm \ref{alg:bcd} and is guaranteed to decrease the loss after each update and so to
converge within a finite number of iterations \cite{horst1996global}. We also note that at
each iteration one needs to compute the equivalent cost matrix $L(\X,\X')\otimes \GG^{(\cdot)}$ which
has a complexity of $O(ndn'd')$. However, one can reduce it using Proposition 1 from
\cite{peyre2016gromov} for the case when $L$ is the squared Euclidean distance $|\cdot|^{2}$ or the Kullback-Leibler divergence. In this case, the overall computational complexity becomes $O(\min\{(n+n')dd'+n'^{2}n;(d+d')nn'+d'^{2}d\})$. In practice, we observed in the numerical experiments that the BCD converges in few iterations (see \textit{e.g.} Figure \ref{fig:loss_bcd}). We refer the interested reader to the \hyperlink{appendix}{supplementary materials} for further details. Finally, we can use the same BCD procedure for the entropic regularized version of \COOT\
\eqref{eq:co-optimal-transport-reg} where each iteration
an entropic regularized OT problem can be solved efficiently
using Sinkhorn's algorithm \cite{cuturi:2013} with several possible improvements
\cite{Altschuler:2017,Altschuler:2019,screekhorn:2019}. Note that this procedure can be easily adapted in the same way to include unbalanced OT problems \cite{chizat_unbalanced} as well.



\section{Relation with other OT distances}
\label{sec:gromov}

\paragraph{Gromov-Wasserstein}
The \COOT\ problem is defined for arbitrary matrices $\X\in \R^{n\times d}, \X'\in
\R^{n'\times d'}$ and so can be readily used to compare pairwise similarity matrices between the samples $\mathbf{C}=\left(c(\xbf_i,\xbf_j)_{i,j}\right)\in \R^{n\times n}, \mathbf{C}'=\left(c'(\xbf_k',\xbf_l')\right)_{k,l} \in \R^{n'\times n'}$ for some $c,c'$. To avoid redundancy, we use the term ``similarity" for both similarity and distance functions in what follows. This situation arises in applications dealing with relational data, \textit{e.g}, in a graph context
\cite{vay2019fgw} or deep metric alignement \cite{gwcnn}. These problems have been
successfully tackled recently using the Gromov-Wasserstein (GW) distance
\cite{memoli_gw} which, given $\mathbf{C}\in \R^{n\times n}$ and $\mathbf{C}'\in \R^{n'\times n'}$, aims at solving: 
\begin{equation}
\label{eq:gromov}
GW(\mathbf{C},\mathbf{C}',\w,\w')=\!\!\!\!\min_{\GGs \in\Pi(\w,\w')} \langle \L(\mathbf{C},\mathbf{C}') \otimes \GGs, \GGs \rangle.
\end{equation}
Below, we explicit the link between GW and \COOT\ using a reduction of a concave QP to an associated BP problem established in \cite{Konno1976} and show that
they are equivalent when working with squared Euclidean distance matrices $\mathbf{C}\in \R^{n\times n}, \mathbf{C}' \in \R^{n'\times n'}$. 

\begin{prop}
\label{concavity_gw_theo}
Let $L=|\cdot|^{2}$ and suppose that $\mathbf{C} \in \R^{n\times n},\mathbf{C}' \in \R^{n'\times n'}$ are squared Euclidean distance matrices such that $\mathbf{C}=\xbf \mathbf{1}_{n}^{T}+\mathbf{1}_{n}\xbf^{T}-2\X\X^{T}, \mathbf{C}'=\xbf' \mathbf{1}_{n'}^{T}+\mathbf{1}_{n'}\xbf'^{T}-2\X'\X'^{T}$ with $\xbf=\text{diag}(\X\X^T),\xbf'=\text{diag}(\X'\X'^T)$. Then, the GW problem can be written as {a concave quadratic program (QP) which Hessian reads} $\mathbf{Q}=-4*\X\X^T \otimes_{K} \X'\X'^T$.


\end{prop}

When working with arbitrary similarity matrices, \COOT\ provides a lower-bound for GW and using Proposition \ref{concavity_gw_theo} we can prove that both problems become equivalent in the Euclidean setting.
\begin{prop}
\label{prop:cot_equal_gw}
Let $\mathbf{C} \in \R^{n\times n},\mathbf{C}' \in \R^{n'\times n'}$ be any symmetric matrices, then: 
$$\COOT(\mathbf{C},\mathbf{C}',\w,\w',\w,\w')\leq GW(\mathbf{C},\mathbf{C}',\w,\w').$$
The converse is also true {under the hypothesis of Proposition \ref{concavity_gw_theo}}. In this case, if $(\GGs_{*},\GGv_{*})$ is an optimal solution of
\eqref{eq:co-optimal-transport}, then both $\GGs_{*},\GGv_{*}$ are solutions of
\eqref{eq:gromov}. Conversely, if $\GGs_{*}$ is an optimal solution of
\eqref{eq:gromov}, then $(\GGs_{*},\GGs_{*})$ is an optimal solution for
\eqref{eq:co-optimal-transport} .
\end{prop}

\begin{wrapfigure}[14]{r}{0.25\linewidth}
    \vspace{-5mm}
\includegraphics[width=1\linewidth]{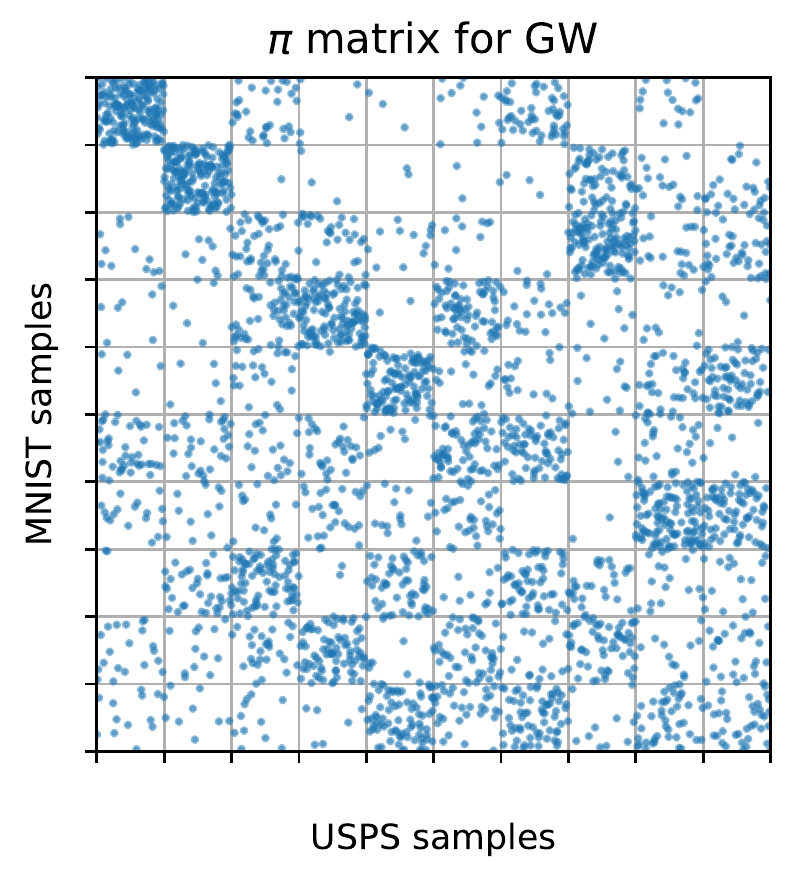}
\caption{{\small GW samples' coupling for MNIST-USPS task}}
\label{fig:gw_coupling_mnist_usps}
\end{wrapfigure}

Under the hypothesis of Proposition \ref{concavity_gw_theo} we know that there exists an optimal solution for the
\COOT\ problem of the form $(\GG_{*},\GG_{*})$, where $\GG_{*}$ is an optimal
solution of the GW problem. This gives a conceptually very simple fixed-point
procedure to compute an optimal solution of GW where one optimises over one coupling only and sets $\GGs_{(k)}=\GGv_{(k)}$ at each iteration of Algorithm \ref{alg:bcd}.
Interestingly enough, in the concave
setting, these iterations are exactly equivalent to
the Frank Wolfe algorithm described in \cite{vay2019fgw} for solving GW. It also corresponds to a Difference of Convex Algorithm (DCA) \cite{tao2005dc,yuille2003concave} where
the concave function is approximated at each iteration by its linear
majorization. When used for entropic
regularized \COOT, the resulting algorithm also recovers exactly the projected gradients
iterations proposed in \cite{peyre2016gromov} for solving the entropic regularized version of GW. {We refer the reader to the \hyperlink{appendix}{supplementary materials} for more details.} 

To conclude, we would like to stress out that \COOT\ is much more than a generalization of GW and that is for multiple reasons. First, it can be used on raw data without requiring to choose or compute the similarity matrices, that can be prohibitively costly, for instance, when dealing with shortest path distances in graphs, and to store them ($O(n^2+n'^2)$ overhead). Second, it can take into account additional information given by feature weights $\v,\v'$ and provides an interpretable mapping between them across two heterogeneous datasets. Finally, contrary to GW, \COOT\ is not invariant neither to feature rotations nor to the change of signs leading to a more informative samples' coupling when compared to GW in some applications. One such example is given in the previous MNIST-USPS transfer task (Figure~\ref{fig:mnist_usps}) for which the coupling matrix obtained via GW (given in Figure~\ref{fig:gw_coupling_mnist_usps}) exhibits important flaws in respecting class memberships when aligning samples.

\paragraph{Invariant OT and Hierarchical OT} {In \cite{alavarez:2019}, the authors proposed InvOT algorithm that aligns samples and learns a transformation between the features of two data matrices given by a linear map with a bounded Schatten p-norm.  
The authors further showed in \cite[Lemma 4.3]{alavarez:2019} that, under some mild assumptions, InvOT and GW lead to the same samples' couplings when cosine similarity matrices are used. It can be proved that, in this case, \COOT\ is also equivalent to them both (see \hyperlink{appendix}{supplementary materials}). However, note that InvOT is applicable  under the strong assumption that $d=d'$ and provides only linear relations between the features, whereas \COOT\ works when $d\neq d'$ and its feature mappings is sparse and more interpretable. InvOT was further used as a building block for aligning clustered datasets in \cite{Hierarchical_ot_2019} where the authors applied it as a divergence measure between the clusters, thus leading to an approach different from ours. Finally, in \cite{YurochkinCCMS19} the authors proposed a hierarchical OT distance as an OT problem with costs defined based on precomputed Wasserstein distances but with no global features' mapping, contrary to \COOT\ that optimises two couplings of the features and the samples simultaneously.}

\section{Numerical experiments}
\label{sec:expes}

In this section, we {highlight} two possible applications of $\COOT$ in a machine learning context: HDA and co-clustering. {We consider these two particular tasks because 1) OT-based methods are considered as a strong baseline in DA; 2) COOT is a natural match for co-clustering as it allows for soft assignments of data samples and features to co-clusters. }
\label{sec:hda}

\subsection{Heterogeneous domain adaptation}

In classification, domain adaptation problem arises when a model learned using a (source)
domain $\X_s = \{\x_i^s\}_{i=1}^{N_s}$ with associated labels $\Y_s =
\{\y_i^s\}_{i=1}^{N_s}$ is to be deployed on a related target domain
$\X_t = \{\x_i^t\}_{i=1}^{N_t}$ where no or only few labelled data are available. 
{Here, we are interested in the {\em heterogeneous} setting where the source and target data belong to different metric spaces. The most prominent works in HDA are based on Canonical
Correlation Analysis~\cite{yeh2014heterogeneous} {and its kernelized version}
and a more recent { approach based on the} Gromov-Wasserstein
distance~\cite{ijcai2018-412}. 
We investigate here the use of  \COOT\ for  both {\em semi-supervised} HDA, where one has access to a small number $n_t$ of labelled samples per class in the target domain and {\em unsupervised} HDA with $n_t=0$.
}

In order to solve the HDA problem, we compute $\COOT(\X_s,\X_t)$ between the two domains and use the $\GGs$
matrix providing a transport/correspondence between samples {(as illustrated in Figure \ref{fig:mnist_usps})} to estimate the
labels in the target domain via label propagation \cite{redko2018optimal}.
Assuming uniform sample weights and one-hot encoded labels, a class prediction $\hat{\Y}_t$ in the target domain samples can be obtained by {computing} $\hat{\Y}_t = \GGs \Y_s$. {When labelled target samples are available, we further prevent source samples to be mapped to target samples from a different class by adding a high cost in the cost matrix for every such source sample as suggested in [Sec. 4.2]\cite{DBLP:journals/pami/CourtyFTR17}. }
\begin{table}[t]
	\begin{center}

	\resizebox{0.9\columnwidth}{!}{
		\begin{tabular}{ccccccc}
		\toprule
			{Domains} & {No-adaptation baseline} & {CCA} & {KCCA} & {EGW} & {SGW} & {\COOT}\\
			\midrule
			C$\rightarrow$W & $69.12$$\pm 4.82$ & $11.47$$\pm 3.78$ & $66.76$$\pm 4.40$ & $11.35$$\pm 1.93$ & $\underline{ 78.88}$$\pm 3.90$ & $\bf 83.47$$\pm 2.60$\\
			W$\rightarrow$C & $83.00$$\pm 3.95$ & $19.59$$\pm 7.71$ & $76.76$$\pm 4.70$ & $11.00$$\pm 1.05$ & $\underline{ 92.41}$$\pm 2.18$ & $\bf 93.65$$\pm 1.80$\\
			W$\rightarrow$W & $82.18$$\pm 3.63$ & $14.76$$\pm 3.15$ & $78.94$$\pm 3.94$ & $10.18$$\pm 1.64$ & $\underline{ 93.12}$$\pm 3.14$ & $\bf 93.94$$\pm 1.84$\\
			W$\rightarrow$A & $84.29$$\pm 3.35$ & $17.00$$\pm 12.41$ & $78.94$$\pm 6.13$ & $7.24$$\pm 2.78$ & $\underline{ 93.41}$$\pm 2.18$ & $\bf 94.71$$\pm 1.49$\\
			A$\rightarrow$C & $\underline{ 83.71}$$\pm 1.82$ & $15.29$$\pm 3.88$ & $76.35$$\pm 4.07$ & $9.82$$\pm 1.37$ & $80.53$$\pm 6.80$ & $\bf 89.53$$\pm 2.34$\\
			A$\rightarrow$W & $81.88$$\pm 3.69$ & $12.59$$\pm 2.92$ & $81.41$$\pm 3.93$ & $12.65$$\pm 1.21$ & $\underline{ 87.18}$$\pm 5.23$ & $\bf 92.06$$\pm 1.73$\\
			A$\rightarrow$A & $\underline{ 84.18}$$\pm 3.45$ & $13.88$$\pm 2.88$ & $80.65$$\pm 3.03$ & $14.29$$\pm 4.23$ & $82.76$$\pm 6.63$ & $\bf 92.12$$\pm 1.79$\\
			C$\rightarrow$C & $67.47$$\pm 3.72$ & $13.59$$\pm 4.33$ & $60.76$$\pm 4.38$ & $11.71$$\pm 1.91$ & $\underline{ 77.59}$$\pm 4.90$ & $\bf 83.35$$\pm 2.31$\\
			C$\rightarrow$A & $66.18$$\pm 4.47$ & $13.71$$\pm 6.15$ & $63.35$$\pm 4.32$ & $11.82$$\pm 2.58$ & $\underline{ 75.94}$$\pm 5.58$ & $\bf 82.41$$\pm 2.79$\\\midrule
			\bf Mean & $78.00$$\pm 7.43$ & $14.65$$\pm 2.29$ & $73.77$$\pm 7.47$ & $11.12$$\pm 1.86$ & $\underline{ 84.65}$$\pm 6.62$ & $\bf 89.47$$\pm 4.74$\\
			\bf p-value & $<$.001 & $<$.001 & $<$.001 & $<$.001 & $<$.001 & -\\
		\bottomrule
		\end{tabular}
		}
	\end{center}
	\caption{{\bf Semi-supervised HDA} for $n_t=3$ from Decaf to GoogleNet task.}

	\label{tab:table_HDA_D_to_G_ns=3}
\end{table}

\paragraph{Competing methods and experimental settings} 
We evaluate COOT on {\em Amazon} (A), {\em Caltech-256} (C) and {\em Webcam} (W) domains from Caltech-Office dataset~\cite{saenko10} with 10 overlapping classes between the domains and two different deep feature representations obtained for images from each domain using the Decaf ~\cite{donahue14} and GoogleNet~\cite{szegedy2015} neural network architectures. In both cases, we extract the image representations as the activations of the last fully-connected layer, yielding respectively sparse 4096 and 1024 dimensional vectors. The heterogeneity comes from these two very
different representations. We consider 4 baselines: CCA, its kernalized version KCCA~\cite{yeh2014heterogeneous} with a Gaussian kernel which
width parameter is set to the inverse of the dimension of the input vector, EGW representing the entropic version of GW and
SGW~\cite{ijcai2018-412} that incorporates labelled target data into two regularization terms. For EGW and SGW, the entropic regularization term was set to $0.1$, and the two other
regularization hyperparameters for the semi-supervised case to $\lambda=10^{-5}$ and $\gamma=10^{-2}$
as done in \cite{ijcai2018-412,Yuguang_ijcai2017}. We use \COOT\ with entropic regularization
on the feature mapping, with parameter $\epsilon_2=1$ in all experiments.  For
all OT methods, we use label propagation to obtain target labels as the maximum entry of $\hat{\Y}_t$ in each row. For all non-OT methods, classification
was conducted with a k-nn classifier with $k=3$. 
We run the experiment in a semi-supervised setting with $n_t=3$, \ie, $3$ samples per class were labelled in the
target domain. The baseline score is the result of classification by only
considering labelled samples in the target domain as the training set.  For each
pair of domains, we selected $20$ samples per class to form the learning sets.
We run this random selection process 10 times and consider the mean accuracy
of the different runs as a performance measure. In the presented results, we perform adaptation from
Decaf to GoogleNet features, and report the results for $n_t \in \{0,1,3,5\}$ in the opposite direction in the \hyperlink{appendix}{supplementary material}. 
\paragraph{Results}
We first provide in Table~\ref{tab:table_HDA_D_to_G_ns=3} the results for the semi-supervised case. From it, we see that \COOT\ surpasses all the other state-of-the-art methods in terms of mean accuracy. This result is confirmed by a $p$-value lower than $0.001$ on a pairwise method comparison with \COOT\ in a Wilcoxon signed rank test. SGW provides the second best result, while CCA and EGW have a less than average performance. Finally, KCCA performs better than the two latter methods, but still fails most of the time to surpass the {no-adaptation baseline score given by a classifier learned on the available labelled target data}. Results for the unsupervised case can be found in
Table~\ref{tab:table_HDA_D_to_G_ns=0}. This setting is rarely considered in the
literature as unsupervised HDA is regarded as a very difficult problem. In this table, we do not provide scores for the no-adaptation baseline and SGW, as they require labelled data. \begin{wraptable}[12]{r}{.5\linewidth}
\vspace{-3mm}
	\resizebox{.5\columnwidth}{!}{
		\begin{tabular}{ccccc}
		\toprule
			{Domains} & {CCA} & {KCCA} & {EGW} & {\COOT}\\
			\midrule
			C$\rightarrow$W & $14.20$$\pm 8.60$ & $\underline{ 21.30}$$\pm 15.64$ & $10.55$$\pm 1.97$ & $\bf 25.50$$\pm 11.76$\\
			W$\rightarrow$C & $13.35$$\pm 3.70$ & $\underline{ 18.60}$$\pm 9.44$ & $10.60$$\pm 0.94$ & $\bf 35.40$$\pm 14.61$\\
			W$\rightarrow$W & $10.95$$\pm 2.36$ & $\underline{ 13.25}$$\pm 6.34$ & $10.25$$\pm 2.26$ & $\bf 37.10$$\pm 14.57$\\
			W$\rightarrow$A & $14.25$$\pm 8.14$ & $\underline{ 23.00}$$\pm 22.95$ & $9.50$$\pm 2.47$ & $\bf 34.25$$\pm 13.03$\\
			A$\rightarrow$C & $11.40$$\pm 3.23$ & $\underline{ 11.50}$$\pm 9.23$ & $11.35$$\pm 1.38$ & $\bf 17.40$$\pm 8.86$\\
			A$\rightarrow$W & $19.65$$\pm 17.85$ & $\underline{ 28.35}$$\pm 26.13$ & $11.60$$\pm 1.30$ & $\bf 30.95$$\pm 18.19$\\
			A$\rightarrow$A & $11.75$$\pm 1.82$ & $\underline{ 14.20}$$\pm 4.78$ & $13.10$$\pm 2.35$ & $\bf 42.85$$\pm 17.65$\\
			C$\rightarrow$C & $12.00$$\pm 4.69$ & $\underline{ 14.95}$$\pm 6.79$ & $12.90$$\pm 1.46$ & $\bf 42.85$$\pm 18.44$\\
			C$\rightarrow$A & $15.35$$\pm 6.30$ & $\underline{ 23.35}$$\pm 17.61$ & $12.95$$\pm 2.63$ & $\bf 33.25$$\pm 15.93$\\\midrule
			\bf Mean & $13.66$$\pm 2.55$ & $\underline{ 18.72}$$\pm 5.33$ & $11.42$$\pm 1.24$ & $\bf 33.28$$\pm 7.61$\\
			\bf p-value & $<$.001 & $<$.001 & $<$.001 & -\\
		\bottomrule
		\end{tabular}
		}
		\caption{{\bf Unsupervised HDA} for $n_t=0$ from Decaf to GoogleNet task.}
		\label{tab:table_HDA_D_to_G_ns=0}
\end{wraptable}
As one can expect, most of the methods fail in obtaining good
classification accuracies in this setting, despite having access to discriminant 
feature representations. Yet, \COOT\ succeeds in providing a meaningful mapping
in some cases. {The overall superior performance of \COOT\ highlights its strengths and underlines the limits of other HDA methods. First, \COOT\ does not depend on approximating empirical quantities from the data, contrary to CCA and KCCA that rely on the estimation of the cross-covariance matrix that is known to be flawed for high-dimensional data with few samples \cite{SongSRH16}. Second, \COOT\ takes into account the features of the raw data that are more informative than the pairwise distances used in EGW. Finally, \COOT\ avoids the sign invariance issue discussed previously that hinders GW's capability to recover classes without supervision as illustrated for the MNIST-USPS problem before.}

\subsection{Co-clustering and data summarization}
While traditional clustering methods present an important discovery tool for data analysis, they discard the relationships that may exist between the features that describe the data samples. \iev{For instance, in recommendation systems, where each user is described in terms of his or her preferences for some product, clustering algorithms may benefit from the knowledge about the correlation between different products revealing their probability of being recommended to the same users.}{} This idea is the cornerstone of \textit{co-clustering} \cite{hartigan-direct-clustering-data-1972} where given a data matrix $\X \in \mathbb{R}^{n\times d}$ and  the number of samples (rows) and features (columns) clusters denoted by $g\leq n$ and $m\leq d$, respectively, we seek to find $\X_c \in \mathbb{R}^{g\times m}$ that summarizes $\X$ in the best way possible. 

\paragraph{\COOT-clustering}
We look for $\X_c$ which is as close as possible to the original $\X$ \textit{w.r.t} \COOT\ by solving $\min_{\X_{c}} \COOT(\X,\X_{c}) =\min_{\GGs,\GGv, \X_{c}} \langle \L(\X,\X_{c}) \otimes \GGs, \GGv \rangle$ with entropic regularization.
More precisely, we set $\w,\w',\v,\v'$ as uniform, initialize $\X_{c}$ with random values and apply the BCD algorithm over ($\GGs,\GGv,\X_{c}$) by alternating between the following steps: 1) obtain
$\GGs$ and $\GGv$ by solving $\COOT(\X,\X_{c})$; 2) set $\X_{c}$ to  $gm\GG^{s\top} \X \GGv$.
{This second step of the procedure is a least-square
estimation when $L=|\cdot|^2$ and corresponds to minimizing the \COOT\ objective
\emph{w.r.t.} $\X_c$. In practice, we observed
that few iterations of this procedure are enough to ensure the convergence. Once solved, we use the soft assignments provided by coupling matrices $\GGs \in \mathbb{R}^{n\times g},\GGv \in \mathbb{R}^{d\times m}$ to partition data points and features to clusters by taking the index of the maximum element in each row of $\GGs$ and $\GGv$, respectively.

\paragraph{Simulated data}
We follow \cite{LaclauRMBB17} where four scenarios with different number of co-clusters, degrees of separation and sizes were considered (for details, see the \hyperlink{appendix}{supplementary materials}). {We choose to evaluate \COOT\ on simulated data as it provides us with the ground-truth for feature clusters that are often unavailable for real-world data sets.} As in \cite{LaclauRMBB17}, we use the same co-clustering baselines including
ITCC \cite{Dhillon:2003:IC:956750.956764}, Double K-Means (DKM)~\cite{rocci_08},
Orthogonal Nonnegative Matrix Tri-Factorizations (ONTMF) \cite{ding_06}, the
Gaussian Latent Block Models (GLBM) \cite{NADI08CI} and Residual Bayesian
Co-Clustering (RBC) \cite{shan_10} as well as the K-means and NMF run on both
modes of the data matrix, as clustering baseline. The performance of all methods
is measured using the co-clustering error (CCE) \cite{Patrikainen06}. 
For all configurations,
we generate 100 data sets and present the mean and standard deviation of the CCE
over all sets for all baselines in Table \ref{tab:result}. Based on these
results, we see that our algorithm outperforms all the other baselines on D1, D2
and D4 data sets, while being behind {\CCOTGW} proposed by \cite{LaclauRMBB17} on
D3. This result is rather strong as our method relies on the original data
matrix, while {\CCOTGW} relies on its kernel
representation and thus benefits from the non-linear information captured by
it. Finally, we note that while both competing methods rely on OT, they remain very different as {\CCOTGW} approach is based on detecting the positions and the number of jumps in the scaling vectors of GW entropic regularized solution, while our method relies on coupling matrices to obtain the partitions.


\paragraph{Olivetti Face dataset} 
As a first application of \COOT\ for the
co-clustering problem on real data, we propose to run the algorithm on the well
known Olivetti faces dataset \cite{samaria1994parameterisation}. 

\begin{figure}
  \centering
  \vspace{-4mm}
  \includegraphics[width=0.8\linewidth]{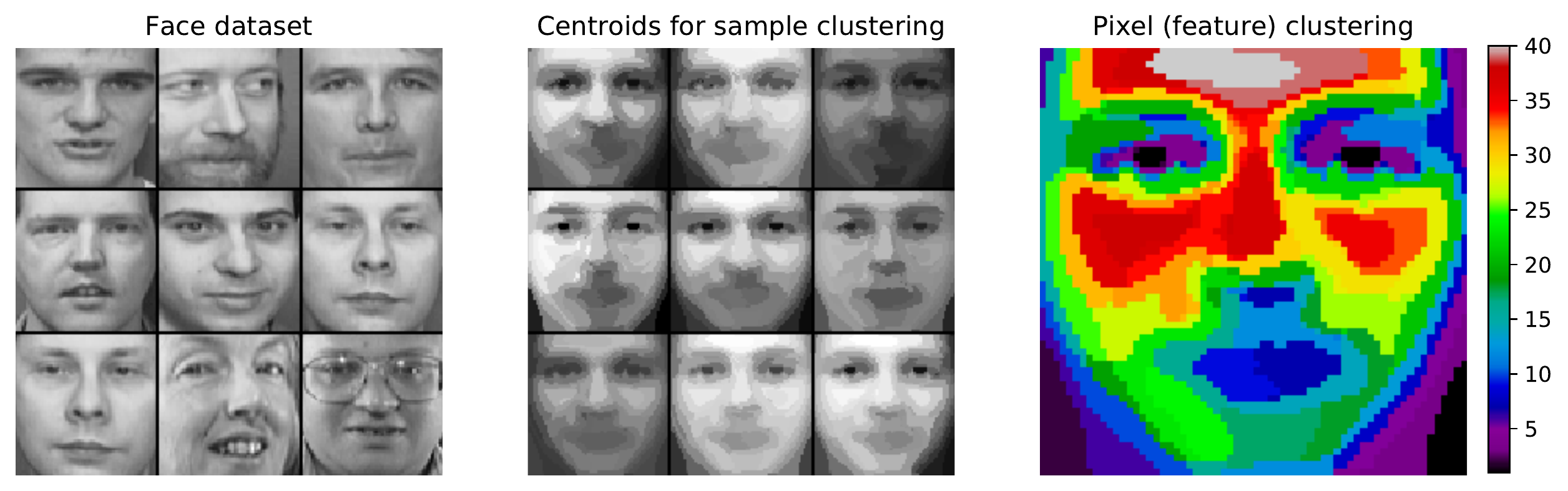}
  \caption{Co-clustering with \COOT\ on the Olivetti faces dataset. \textbf{(left)} Example images from the dataset, \textbf{(center)} centroids estimated by \COOT\, \textbf{(right)} clustering of the pixels estimated by \COOT\ where each color represents a cluster.
  }
  \label{fig:coclust_faces}
\end{figure}

We take 400
images normalized between 0 and 1 and run our algorithm with $g=9$ image
clusters and $m=40$ feature (pixel) clusters. {As before, we consider the empirical distributions supported on images and features, respectively.}  The resulting reconstructed image's
clusters are given in Figure
\ref{fig:coclust_faces} and the pixel clusters are illustrated in its rightmost part. We can see that despite the high variability in the data set, we still
manage to recover detailed centroids, {whereas} L2-based clustering {such as standard NMF or k-means based on $\ell_2$ norm cost function are known to
provide blurry estimates in this case. Finally, as in the MNIST-USPS example, \COOT\ recovers spatially
localized pixel clusters with no prior information about the pixel relations.
\begin{table*}[!t]
\begin{center}
\resizebox{\linewidth}{!}{
\begin{tabular}{lcccccccccc}
\hline
\multirow{2}{*}{Data set} &\multicolumn{10}{c}{Algorithms}\\
\cline{2-11}
  & K-means&NMF&DKM&Tri-NMF&GLBM&ITCC&RBC&\CCOT&\CCOTGW & COOT\\
\hline
D1&$.018\pm{.003}$&$.042\pm{.037}$&$.025\pm{.048}$&$.082\pm{.063}$&$.021\pm{.011}$&$.021\pm{.001}$&$.017\pm{.045}$&$.018\pm{.013}$&$.004\pm{.002}$ & $\mathbf{0}$ \\
D2&$.072\pm{.044}$&$.083\pm{.063}$&$.038\pm{.000}$&$.052\pm{.065}$&$.032\pm{.041}$&$.047\pm{.042}$&$.039\pm{.052}$&$.023\pm{.036}$&$.011\pm{.056}$ & $\mathbf{.009\pm{0.04}}$\\
D3&--&--&$.310\pm{.000}$&--&$.262\pm{.022}$&$.241\pm{.031}$&--&$.031\pm{.027}$&$\mathbf{.008\pm{.001}}$ & $.04\pm{.05}$\\
D4&$.126\pm{.038}$&--&$.145\pm{.082}$&--&$.115\pm{.047}$&$.121\pm{.075}$&$.102\pm{.071}$&$.093\pm{.032}$&$.079\pm{.031}$ & $\mathbf{0.068\pm{0.04}}$\\
\hline
\end{tabular}
}
\end{center}
\caption{\label{tab:result} Mean ($\pm$ standard-deviation) of the co-clustering error (CCE) obtained for all configurations. ``-" indicates that the algorithm cannot find a partition with the requested number of co-clusters. All the baselines results (first 9 columns) are from \cite{LaclauRMBB17}.
}
\end{table*}

\paragraph{MovieLens} We now evaluate our approach on the benchmark \MovieL-100K\footnote{https://grouplens.org/datasets/movielens/100k/} data set that provides 100,000 user-movie ratings, on a scale of one to five, collected from 943 users on 1682 movies. The main goal of our algorithm here is to summarize the initial data matrix so that $\X_c$ reveals the blocks (co-clusters) of movies and users that share similar tastes. We set the number of user and film clusters to $g=10$ and $m=20$, respectively as in \cite{Banerjee:2007:GME:1314498.1314563}. 

\begin{table}[t!]
\centering
\resizebox{0.67\linewidth}{!}{
\begin{tabular}{cc}
\hline
M1 &M20\\
\hline
Shawshank Redemption (1994)& Police Story 4: Project S (Chao ji ji hua) (1993)\\
Schindler's List (1993) & Eye of Vichy, The (Oeil de Vichy, L') (1993) \\
Casablanca (1942) & Promise, The (Versprechen, Das) (1994)\\
Rear Window (1954) & To Cross the Rubicon (1991)\\
Usual Suspects, The (1995) & Daens (1992)\\
\hline
\end{tabular}
}
\vspace{3mm}
\caption{\label{tab:top5} Top 5 of movies in clusters M1 and M20. Average rating of the top 5 rated movies in M1 is 4.42, while for the M20 it is 1.}
\end{table}
The obtained results provide the first movie cluster consisting of
films with high ratings (3.92 on average), while the last movie cluster includes
movies with very low ratings (1.92 on average). Among those, we show the 5
best/worst rated movies in those two clusters in Table \ref{tab:top5}. Overall, our algorithm manages to find a coherent co-clustering structure in \MovieL-100K and obtains results similar to those provided in \cite{LaclauRMBB17,Banerjee:2007:GME:1314498.1314563}.


\section{Discussion and conclusion}

\label{sec:conclu}

In this paper, we presented a novel optimal transport problem which aims at comparing 
distributions supported on samples belonging to different spaces. To this end, two optimal transport maps, one acting on the sample space, 
and the other on the feature space, are optimized to connect the two heterogeneous distributions. We provide several algorithms allowing to solve it in general and special cases and show its connections to other OT-based problems. We further demonstrate its
usefulness and versatility on two difficult machine learning problems: heterogeneous domain adaptation and 
co-clustering/data summarization, where promising results were obtained. Numerous follow-ups of this work are expected. Beyond the potential applications of the method in various contexts, such as {\em e.g.} statistical matching, data analysis or even losses in deep learning settings, one immediate and intriguing question lies into the generalization of this framework in the continuous setting, and  the potential connections to duality theory. This might lead to stochastic optimization schemes  enabling large scale solvers for this problem. 

\subsection*{Acknowledgements}

We thank L\'eo Gautheron, Guillaume Metzler and Raphaël Chevasson for proofreading the manuscript before the submission. This work benefited from the support from  OATMIL ANR-17-CE23-0012 project of the French National Research Agency (ANR). This work has been supported by the French government, through the 3IA Côte d’Azur Investments in the Future project managed by the National Research Agency (ANR) with the reference number ANR-19-P3IA-0002.
This action benefited from the support of the Chair "Challenging Technology for Responsible Energy" led by l'X – Ecole polytechnique and the Fondation de l’Ecole polytechnique, sponsored by TOTAL.
 We gratefully acknowledge the support of NVIDIA Corporation with the donation of the Titan X GPU used for this research.

\section*{Broader impact}
Despite its evident usefulness the problem of finding the correspondences between two datasets is rather general and may arise in many fields in machine learning. Consequently it is quite difficult to exhaustively state all the potential negative ethical impacts that may occur when using our method. As described in the paper, it could be used to solve the so-called election isomorphism problem \cite{faliszewski19} where one wants to find how similar are two elections based on the knowledge of votes and candidates. Although having these type of datasets seems unrealistic in modern democracies, using our approach on this problem runs the risk of breaking some privacy standards by revealing precisely how the votes have been moved from one election to the other. Generally speaking, and when given access to two datasets with sensitive data, our method is able to infer correspondences between instances \emph{and} features which could possibly lead to privacy issues for a malicious user. From a different perspective, the Optimal Transport framework is known to be quite computationally expensive and even recent improvements turns out to be super-linear in terms of the computational complexity. It is not an energy-free tool and in a time when carbon footprints must be drastically reduced, one should have in mind the potential negative impact that computationally demanding algorithms might have on the planet.






\newpage

\Large{\hypertarget{appendix}{\textbf{Supplementary materials}}}%

\normalsize
\setcounter{section}{0}
\setcounter{prop}{0}
\def\thesection{\Alph{section}}

The supplementary is organized as follows. After the MNIST-USPS illustration (Section 2 of the main paper), Section \ref{sec:properties} presents the proof of Proposition \ref{prop:distance} from the main paper and the computational complexity of calculating the value of the \COOT\ problem as mentioned in Section 2.3 of the main paper.  We provide the proofs for the equivalence of \COOT\ to Gromov-Wassserstein distance (Propositions \ref{concavity_gw_theo} and \ref{prop:cot_equal_gw} from the main paper and algorithmic implications discussed after Proposition 3), InvOT and election isomorphism problem in Section \ref{sec:proofs3}. Finally, in Section \ref{sec:expes_supp}, we provide additional experimental results for heterogeneous domain adaptation problem and precise the simulation details for the co-clustering task. 

\section{Illustration on MNIST-USPS task}
We provide a comparison between the coupling matrices obtained using GW and \COOT\ on the MNIST-USPS problem from Section 2 of the main paper in Figure \ref{fig:cooot_vs_gw} and show the results of transporting the USPS samples to MNIST and vice versa using barycentric mapping in Figures \ref{fig:usps_to_mnist_piv} and \ref{fig:mnist_to_usps_piv}.
\begin{figure*}[!h]
	\centering
	\includegraphics[width=0.9\linewidth]{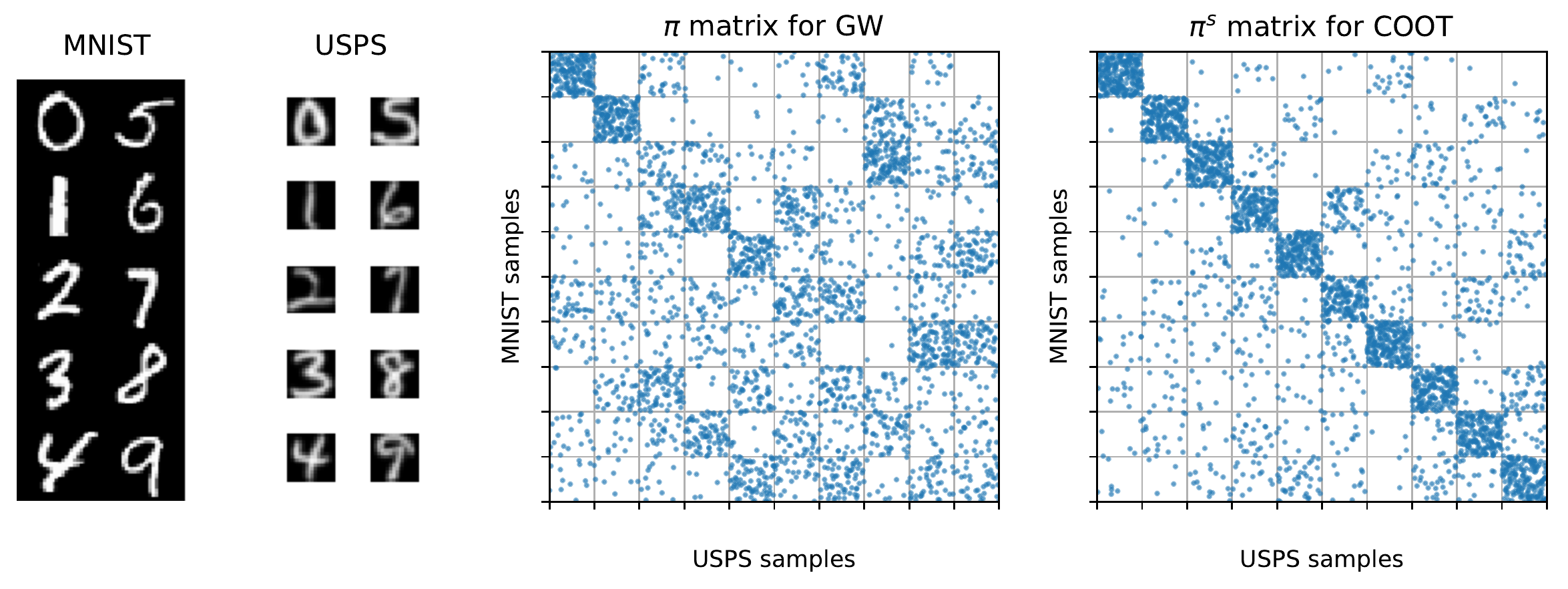}
	\caption{Comparison between the coupling matrices obtained via GW and COOT on MNIST-USPS.}
	\label{fig:cooot_vs_gw}
\end{figure*}
\begin{figure*}[!h]
    \centering
    \includegraphics[width=.9\linewidth]{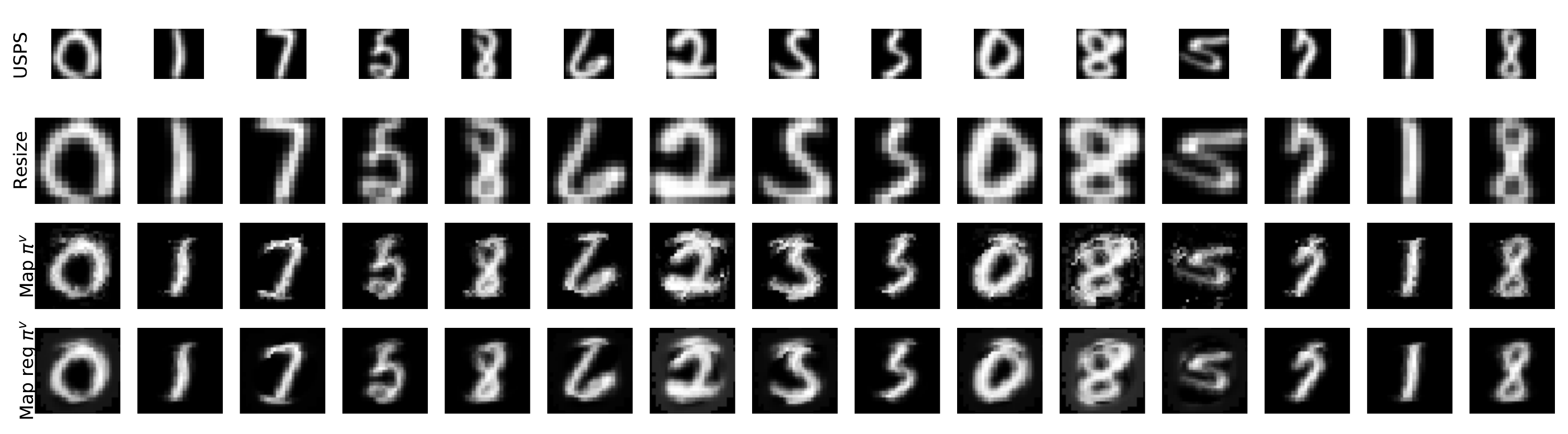}
    \caption{Linear mapping from USPS to MNIST using $\GGv$.  \textbf{(First row)} Original USPS samples,
    \textbf{(Second row)} Samples resized to target resolution, \textbf{(Third row)} Samples mapped using $\GGv$, \textbf{(Fourth row)} Samples mapped using $\GGv$ with entropic regularization.}
    \label{fig:usps_to_mnist_piv}
\end{figure*}
\begin{figure*}[!h]
    \centering
    \includegraphics[width=.9\linewidth]{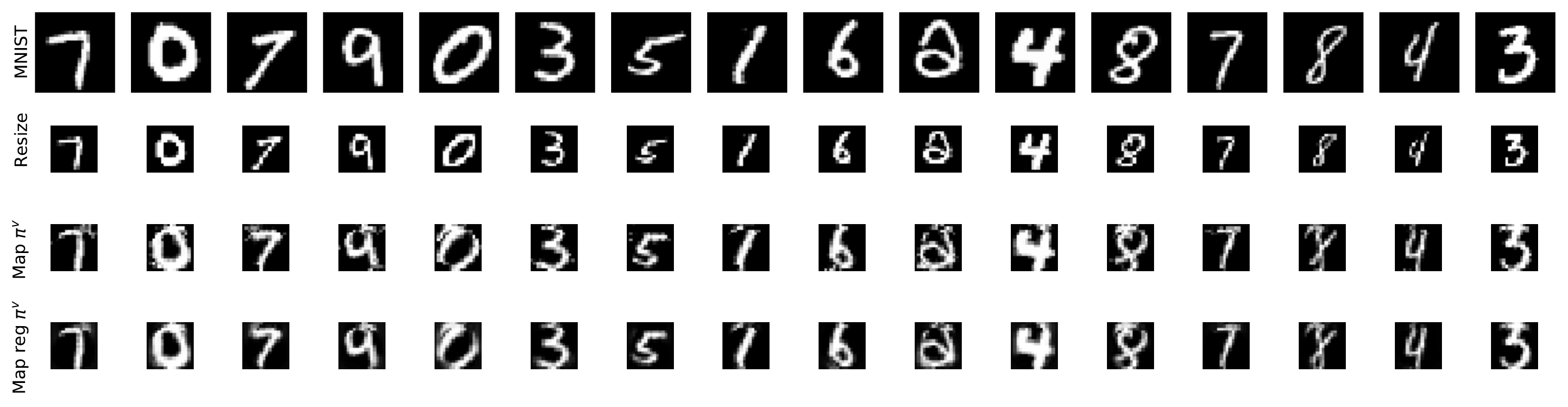}
    \caption{Linear mapping from MNIST to USPS using $\GGv$.  \textbf{(First row)} Original MNIST samples,
    \textbf{(Second row)} Samples resized to target resolution, \textbf{(Third row)} Samples mapped using $\GGv$, \textbf{(Fourth row)} Samples mapped using $\GGv$ with entropic regularization.}
    \label{fig:mnist_to_usps_piv}
\end{figure*}

\section{Proofs from Section 2}
\label{sec:properties}
\subsection{Proof of Proposition 1}
\begin{prop}[\COOT\ is a distance {for $n=n',d=d'$}]
\label{prop:distance}
Suppose $L=|\cdot|^{p}, p \geq 1$, $n=n',d=d'$ and that the weights $\w,\w',\v,\v'$ are uniform. Then $\COOT(\X,\X')=0$ \textit{iff} there exists a permutation of the samples $\sigma_{1} \in \mathbb{S}_{n}$ and of the features $\sigma_{2} \in \mathbb{S}_{d}$, \textit{s.t}, $\forall i,k \ \X_{i,k}=\X'_{\sigma_{1}(i),\sigma_{2}(k)}$. 
Moreover, in general for $n\neq n'$, $d\neq d'$ and potentially non uniform weights, it is symmetric and satisfies the triangle inequality as long as $L$ satisfies the triangle inequality $\COOT(\X,\X'')\leq \COOT(\X,\X')+\COOT(\X',\X'')$. 
\end{prop}
\begin{proof}
The symmetry follows from the definition of \COOT. To prove the triangle inequality of \COOT\ for arbitrary measures, we will use the gluing lemma (see \cite{Villani}) which states the existence of couplings with a prescribed structure. 
Let $\X \in\mathbb{R}^{n\times d},\X' \in\mathbb{R}^{n'\times d'},\X'' \in\mathbb{R}^{n''\times d''}$ associated with $\w \in \Delta_n,\v \in \Delta_d,\w' \in \Delta_n',\v' \in \Delta_d',\w'' \in \Delta_n'',\v'' \in \Delta_d''$. Without loss of generality, we can suppose in the proof that all weights are different from zeros (otherwise we can consider $\tilde{w}_{i}=w_{i}$ if $w_{i}>0$ and $\tilde{w}_{i}=1$ if $w_{i}=0$ see proof of Proposition 2.2 in \cite{cot_peyre_cutu})

Let $(\GGs_{1},\GGv_{1})$ and $(\GGs_{2},\GGv_{2})$ be two couples of optimal solutions for the $\COOT$ problems associated with $\COOT(\X,\X',\w,\w',\v,\v')$ and $\COOT(\X',\X'',\w',\w'',\v',\v'')$ respectively.

We define:
\begin{equation*}
S_{1}=\GGs_{1}\text{diag}\left(\frac{1}{\w'}\right)\GGs_{2}, \quad
S_{2}=\GGv_{1}\text{diag}\left(\frac{1}{\v'}\right)\GGv_{2}
\end{equation*}

Then, it is easy to check that $S_{1} \in \Pi(\w,\w'')$ and $S_{2} \in \Pi(\v,\v'')$ (see \textit{e.g} Proposition 2.2 in \cite{cot_peyre_cutu}). We now show the following:
\begin{equation*}
\begin{split}
&\COOT(\X,\X'',\w,\w'',\v,\v'') \stackrel{*}{\leq} \langle \mathbf{L}(\X,\X'')\otimes S_{1}, S_{2} \rangle = \langle \mathbf{L}(\X,\X'')\otimes [\GGs_{1}\text{diag}(\frac{1}{\w'})\GGs_{2}], [\GGv_{1}\text{diag}(\frac{1}{\v'})\GGv_{2}] \rangle \\
&\stackrel{**}{\leq} \langle [\mathbf{L}(\X,\X')+\mathbf{L}(\X',\X'')]\otimes [\GGs_{1}\text{diag}(\frac{1}{\w'})\GGs_{2}], [\GGv_{1}\text{diag}(\frac{1}{\v'})\GGv_{2}] \rangle \\
&=\langle \mathbf{L}(\X,\X')\otimes [\GGs_{1}\text{diag}(\frac{1}{\w'})\GGs_{2}], [\GGv_{1}\text{diag}(\frac{1}{\v'})\GGv_{2}] \rangle +\langle \mathbf{L}(\X',\X'')\otimes [\GGs_{1}\text{diag}(\frac{1}{\w'})\GGs_{2}], [\GGv_{1}\text{diag}(\frac{1}{\v'})\GGv_{2}] \rangle,
\end{split}
\end{equation*}
\noindent where in (*) we used the suboptimality of $S_{1},S_{2}$ and in (**) the fact that $L$ satisfies the triangle inequality.

Now note that:
\begin{equation*}
\begin{split}
&\langle \mathbf{L}(\X,\X')\otimes [\GGs_{1}\text{diag}(\frac{1}{\w'})\GGs_{2}], [\GGv_{1}\text{diag}(\frac{1}{\v'})\GGv_{2}] \rangle +\langle \mathbf{L}(\X',\X'')\otimes [\GGs_{1}\text{diag}(\frac{1}{\w'})\GGs_{2}], [\GGv_{1}\text{diag}(\frac{1}{\v'})\GGv_{2}] \rangle\\
&= \sum_{i,j,k,l,e,o} L(X_{i,k},X'_{e,o}) \frac{\GGs_{1}{_{i,e}} \GGs_{2}{_{e,j}}}{w'_{e}} \frac{\GGv_{1}{_{k,o}} \GGv_{2}{_{o,l}}}{v'_{o}} +\sum_{i,j,k,l,e,o} L(X'_{e,o},X''_{j,l}) \frac{\GGs_{1}{_{i,e}} \GGs_{2}{_{e,j}}}{w'_{e}} \frac{\GGv_{1}{_{k,o}} \GGv_{2}{_{o,l}}}{v'_{o}} \\
&\stackrel{*}{=} \sum_{i,k,e,o} L(X_{i,k},X'_{e,o})\GGs_{1}{_{i,e}} \GGv_{1}{_{k,o}} +\sum_{l,j,e,o} L(X'_{e,o},X''_{j,l}) \GGs_{2}{_{e,j}} \GGv_{2}{_{o,l}}
\end{split}
\end{equation*}
where in (*) we used:
\begin{equation*}
\sum_{j} \frac{\GGs_{2}{_{e,j}}}{w'_{e}}=1, \ \sum_{l} \frac{\GGv_{2}{_{o,l}}}{v'_{o}}=1, \ \sum_{i} \frac{\GGs_{1}{_{i,e}}}{w'_{e}}=1, \ \sum_{k} \frac{\GGv_{1}{_{k,o}}}{v'_{o}}=1
\end{equation*}
Overall, from the definition of $\GGs_{1},\GGv_{1}$ and $\GGs_{2},\GGv_{2}$ we have:
\begin{equation*}
\begin{split}
&\COOT(\X,\X'',\w,\w'',\v,\v'')\leq  \COOT(\X,\X',\w,\w',\v,\v')+\COOT(\X',\X'',\w',\w'',\v',\v'').
\end{split}
\end{equation*}

For the identity of indiscernibles, suppose that $n=n',d=d'$ and that the weights $\w,\w',\v,\v'$ are uniform. Suppose that there exists a permutation of the samples $\sigma_{1} \in \mathbb{S}_{n}$ and of the features $\sigma_{2} \in \mathbb{S}_{d}$, \textit{s.t} $\forall i,k \in [\![n]\!]\times[\![d]\!], \ \X_{i,k}=\X'_{\sigma_{1}(i),\sigma_{2}(k)}$. We define the couplings $\pi^{s},\pi^{v}$ supported on the graphs of the permutations $\sigma_1,\sigma_2$ respectively, \textit{i.e} $\pi^{s}=(Id \times \sigma_{1})$ and $\pi^{v}=(Id \times \sigma_{2})$. These couplings have the prescribed marginals and lead to a zero cost hence are optimal.

Conversely, as described in the paper, there always exists an optimal solution of \eqref{eq:co-optimal-transport} which lies on extremal points of the polytopes $\Pi(\w,\w')$ and $\Pi(\v,\v')$. When $n=n',d=d'$ and uniform weights are used, Birkhoff’s theorem \cite{birkhoff:1946} states that the set of extremal points of $\Pi(\frac{\one_{n}}{n},\frac{\one_{n}}{n})$ and $\Pi(\frac{\one_{d}}{d},\frac{\one_{d}}{d})$ are the set of permutation matrices so there exists an optimal solution $(\GGs_{*},\GGv_{*})$ supported on $\sigma_{*}^{s},\sigma_{*}^{v}$ respectively with $\sigma_{*}^{s},\sigma_{*}^{v} \in \mathbb{S}_{n} \times \mathbb{S}_{d}$. 
Then, if $\COOT(\X,\X')=0$, it implies that $\sum_{i,k} L(X_{i,k},X'_{\sigma_{*}^{s}(i),\sigma_{*}^{v}(k)})=0$. If $L=|\cdot|^{p}$ then $X_{i,k}=X'_{\sigma_{*}^{s}(i),\sigma_{*}^{v}(k)}$ which gives the desired result. If $n\neq n',d\neq d'$ the \COOT\ cost is always strictly positive as there exists a strictly positive element outside the diagonal.

\end{proof}

\subsection{Complexity of computing the value of \COOT}
\label{sec:complexity}
As mentionned in \cite{peyre2016gromov}, if $L$ can be written as $L(a,b)=f(a)+f(b)-h_{1}(a)h_{2}(b)$ then we have that
$$\mathbf{L}(\X,\X')\otimes \GGs=\mathbf{C}_{\X,\X'}-h_{1}(\X) \GGs h_{2}(\X')^{T},$$
where $\mathbf{C}_{\X,\X'}=\X \w \mathbbm{1}_{n'}^{T}+\mathbbm{1}_{n} \w'^{T}\X'^{T}$ so that the latter can be computed in $O(ndd'+n'dd')=O((n+n')dd')$. To compute the final cost, we must also calculate the scalar product with $\GGv$ that can be done in $O(n'^{2}n)$ making the complexity of $\langle \mathbf{L}(\X,\X')\otimes \GGs, \GGv \rangle$ equal to $O((n+n')dd'+n'^{2}n)$. 

Finally, as the cost is symmetric \textit{w.r.t } $\GGs,\GGv$, we obtain the overall complexity of $O(\min\{(n+n')dd'+n'^{2}n;(d+d')nn'+d'^{2}d\})$.

\section{Proofs from Section 3}
\label{sec:proofs3}
\subsection{Equivalence between BAP and QAP}
As pointed in \cite{Konno1976}, we can relate the solutions of a QAP and a BAP using the following theorem:
\begin{theo}
\label{equivalence_theo}
If $\mathbf{Q}$ is a positive semi-definite matrix, then problems: 
\begin{equation}
\label{qap2}
\begin{array}{cl}{\max _{\xbf} f(\xbf)} & {=\mathbf{c}^{T} \xbf+\frac{1}{2} \xbf^{T} \mathbf{Q} \xbf} \\ {\text {s.t.}} & {\mathbf{A} \xbf = \mathbf{b}},\;  {\xbf \geq 0}
\end{array}
\end{equation}
\begin{equation}
\label{bilinearqap2}
\begin{array}{cl}{\max _{\xbf, \ybf} g(\xbf, \ybf)} & {=\frac{1}{2}\mathbf{c}^{T} \xbf+\frac{1}{2} \mathbf{c}^{T}\ybf+\frac{1}{2} \xbf^{T} \mathbf{Q} \ybf} \\ {\text {s.t.}} & {\mathbf{A} \xbf = \mathbf{b}, \mathbf{A} \ybf =\mathbf{b}},\;   {\xbf, \ybf \geq 0}
\end{array}
\end{equation}
are equivalent. More precisely, if $\xbf^{*}$ is an optimal solution for \eqref{qap2}, then $(\xbf^{*},\xbf^{*})$ is a solution for \eqref{bilinearqap2} and if $(\xbf^{*},\ybf^{*})$ is optimal for \eqref{bilinearqap2}, then both $\xbf^{*}$ and $\ybf^{*}$ are optimal for \eqref{qap2}.
\end{theo}

\begin{proof}
This proof follows the proof of Theorem 2.2 in \cite{Konno1976}. Let $\zbf^{*}$ be optimal for \eqref{qap2} and $(\xbf^{*},\ybf^{*})$ be optimal for \eqref{bilinearqap2}. Then, by definition, for all $\xbf$ satisfying the constraints of \eqref{qap2}, $f(\zbf^{*})\geq f(\xbf)$. In particular, $f(\zbf^{*})\geq f(\xbf^{*})=g(\xbf^{*},\xbf^{*})$ and $f(\zbf^{*})\geq f(\ybf^{*})=g(\ybf^{*},\ybf^{*})$. Also, $g(\xbf^{*},\ybf^{*})\geq \max_{\xbf,\xbf \ \text{s.t} \ \mathbf{A} \xbf = \mathbf{b}, \xbf\geq 0} g(\xbf,\xbf)=f(\zbf^{*})$.

To prove the theorem, it suffices to prove that 
\begin{equation}
\label{toprove}
f(\ybf^{*})=f(\xbf^{*})=g(\xbf^{*},\ybf^{*}) 
\end{equation}
since, in this case, $g(\xbf^{*},\ybf^{*})=f(\xbf^{*})\geq f(\zbf^{*})$ and $g(\xbf^{*},\ybf^{*})=f(\ybf^{*})\geq f(\zbf^{*})$.

Let us prove \eqref{toprove}. Since $(\xbf^{*},\ybf^{*})$ is optimal, we have:
\begin{equation*}
\begin{split}
 0\leq g(\xbf^{*},\ybf^{*})-g(\xbf^{*},\xbf^{*})&= \frac{1}{2} \mathbf{c}^{T}(\ybf^{*}-\xbf^{*}) + \frac{1}{2} {\xbf^{*}}^{T} \mathbf{Q} (\ybf^{*}-\xbf^{*})\\
0\leq g(\xbf^{*},\ybf^{*})-g(\ybf^{*},\ybf^{*})&= \frac{1}{2} \mathbf{c}^{T}(\xbf^{*}-\ybf^{*}) + \frac{1}{2} {\ybf^{*}}^{T} \mathbf{Q} (\xbf^{*}-\ybf^{*}).
\end{split}
\end{equation*}

By adding these inequalities we obtain:
\begin{equation*}
(\xbf^{*}-\ybf^{*})^{T} \mathbf{Q} (\xbf^{*}-\ybf^{*})\leq 0.
\end{equation*}

Since $\mathbf{Q}$ is positive semi-definite, this implies that $\mathbf{Q} (\xbf^{*}-\ybf^{*})=0$. So, using previous inequalities, we have $\mathbf{c}^{T}(\xbf^{*}-\ybf^{*})=0$, hence $g(\xbf^{*},\ybf^{*})=g(\xbf^{*},\xbf^{*})=g(\ybf^{*},\ybf^{*})$ as required. 

Note also that this result holds when we add a constant term to the cost function. 
\end{proof}

\subsection{Proofs of Propositions 2 and 3}
We now prove all the theorems from Section 3 from the main paper. We first recall the GW problem for two matrices $\mathbf{C},\mathbf{C}'$:
\begin{equation}
\label{eq:gromov_sup}
GW(\mathbf{C},\mathbf{C}',\w,\w')=\min_{\GGs \in\Pi(\w,\w')} \langle L(\mathbf{C},\mathbf{C}') \otimes \GGs, \GGs \rangle.
\end{equation}

We will now prove the Proposition 2 in the main paper stated as follows.
\begin{prop}
Let $L=|\cdot|^{2}$ and suppose that $\mathbf{C} \in \R^{n\times n},\mathbf{C}' \in \R^{n'\times n'}$ are squared Euclidean distance matrices such that $\mathbf{C}=\xbf \mathbf{1}_{n}^{T}+\mathbf{1}_{n}\xbf^{T}-2\X\X^{T}, \mathbf{C}'=\xbf' \mathbf{1}_{n'}^{T}+\mathbf{1}_{n'}\xbf'^{T}-2\X'\X'^{T}$ with $\xbf=\text{diag}(\X\X^T),\xbf'=\text{diag}(\X'\X'^T)$. Then, the GW problem can be written as {a concave quadratic program (QP) which Hessian reads} $\mathbf{Q}=-4*\X\X^T \otimes_{K} \X'\X'^T$.
\end{prop}

This result is a consequence of the following lemma.
\begin{lemma}
\label{concavity_gw_theo_sup}
With previous notations and hypotheses, the GW problem can be formulated as:
\begin{align*}
GW(\mathbf{C},\mathbf{C}',\w,\w')&=\min_{\GGs \in\Pi(\w,\w')} -4\vec(\mathbf{M})^{T}\vec(\GGs) -8\vec(\GGs)^{T}\mathbf{Q}\vec(\GGs) +Cte  
\end{align*}
with
\begin{equation*}
\begin{split}
&\mathbf{M}=\xbf\xbf'^T-2\xbf\w'^{T}\X'\X'^T-2\X\X^T\w \xbf'^{T} \text{ and } \mathbf{Q}=\X\X^T \otimes_{K} \X'\X'^T,\\
&Cte=\sum_{i}\|\xbf_i-\xbf_j\|_{2}^{4} \w_i \w_j + \sum_{i}\|\xbf'_i-\xbf'_j\|_{2}^{4} \w'_i \w'_j -4 \w^{T}\xbf\w'^{T}\xbf'
\end{split}
\end{equation*}
 
\end{lemma}
\begin{proof}
Using the results in \cite{peyre2016gromov} for $L=|\cdot|^{2}$, we have $\mathbf{L}(\mathbf{C},\mathbf{C}') \otimes \GGs=c_{\mathbf{C},\mathbf{C}'}-2\mathbf{C}\GGs \mathbf{C}'$ with $c_{\mathbf{C},\mathbf{C}'}=(\mathbf{C})^{2}\w\one_{n'}^{T}+\one_{n}\w'^{T}(\mathbf{C}')^{2}$, where $(\mathbf{C})^{2}=(\mathbf{C}_{i,j}^{2})$ is applied element-wise.


We now have that
\begin{equation*}
\begin{split}
 &\langle \mathbf{C}\GGs \mathbf{C}', \GGs \rangle =\tr\big[{\GGs}^{T}(\xbf \one_{n}^{T}+\one_{n}\xbf^{T}-2\X\X^T) \GGs (\xbf' \one_{n'}^{T}+\one_{n'}\xbf'^{T}-2\X'\X'^T) \big] \\
 &=\tr\big[({\GGs}^{T}\xbf\one_{n}^{T}+\w'\xbf^{T}-2{\GGs}^{T}\X\X^T)(\GGs \xbf'\one_{n'}^{T}+\w \xbf'^{T}-2\GGs \mathbf{X}'\mathbf{X}'^T) \big]\\
 &=\tr\big[{\GGs}^{T}\xbf\w'^{T}\xbf'\one_{n'}^{T}+{\GGs}^{T}\xbf\xbf'^{T}-2{\GGs}^{T}\xbf\w'^{T}\mathbf{X}'\mathbf{X}'^T +\w'\xbf^{T}\GGs \xbf' \one_{n'}^{T}+\w'\xbf^{T}\w \xbf'^{T} - 2 \w'\xbf^{T}\GGs \mathbf{X}'\mathbf{X}'^T \\
 &-2 {\GGs}^{T} \X\X^T\GGs \xbf' \one_{n'}^{T} -2 {\GGs}^{T}\X\X^T\w \xbf'^{T} +4 {\GGs}^{T} \X\X^T\GGs \mathbf{X}'\mathbf{X}'^T\big] \\
 &\stackrel{*}{=}\tr\big[{\GGs}^{T}\xbf\w'^{T}(\xbf'\one_{n'}^{T}+\one_{n'}\xbf'^{T})+{\GGs}^{T}\xbf\xbf'^{T}+\w'\xbf^{T}\w \xbf'^{T}-2{\GGs}^{T}\xbf\w'^{T}\mathbf{X}'\mathbf{X}'^T-2\w'\xbf^{T}\GGs \mathbf{X}'\mathbf{X}'^T\\
 &-2 {\GGs}^{T} \X\X^T\GGs \xbf' \one_{n'}^{T} -2 {\GGs}^{T}\X\X^T\w \xbf'^{T} +4 {\GGs}^{T} \X\X^T\GGs \mathbf{X}'\mathbf{X}'^T\big],
 \end{split}
\end{equation*}
where in (*) we used:
\begin{equation*}
\begin{split}
\tr(\w'\xbf^{T}\GGs \xbf' \one_{n'}^{T})=\tr(\xbf'\one_{n'}^{T}\w'\xbf^{T}\GGs)=\tr({\GGs}^{T}\xbf\w'^{T}\one_{n'}\xbf'^{T}).
 \end{split}
\end{equation*}

Moreover, since:
\begin{equation*}
\begin{split}
&\tr({\GGs}^{T}\X\X^T\GGs \xbf' \one_{n'}^{T})=\tr(\one_{n'}^{T}{\GGs}^{T}\X\X^T\GGs \xbf')=\tr(\w^{T}\X\X^T\GGs \xbf')=\tr({\GGs}^{T}\X\X^T\w \xbf'^{T})
 \end{split}
\end{equation*}
and $\tr(\w'\xbf^{T}\GGs \mathbf{X}'\mathbf{X}'^T)=\tr({\GGs}^{T}\xbf\w'^{T}\mathbf{X}'\mathbf{X}'^T)$, we can simplify the last expression to obtain:
\begin{equation*}
\begin{split}
 &\langle \mathbf{C}\GGs \mathbf{C}', \GGs \rangle =\tr\big[{\GGs}^{T}\xbf\w'^{T}(\xbf'\one_{n'}^{T}+\one_{n'}\xbf'^{T})+{\GGs}^{T}\xbf\xbf'^{T}+\w'\xbf^{T}\w \xbf'^{T} \\
 &-4{\GGs}^{T}\xbf\w'^{T}\mathbf{X}'\mathbf{X}'^T-4{\GGs}^{T}\X\X^T\w \xbf'^{T}+4 {\GGs}^{T} \X\X^T\GGs \mathbf{X}'\mathbf{X}'^T \big].
 \end{split}
\end{equation*}

Finally, we have that
\begin{equation*}
\begin{split}
 &\langle \mathbf{C}\GGs \mathbf{C}', \GGs \rangle =\tr\big[{\GGs}^{T}\xbf\w'^{T}\xbf'\one_{n'}^{T}+{\GGs}^{T}\xbf\w'^{T}\one_{n'}\xbf'^{T}+{\GGs}^{T}\xbf\xbf'^{T}\\
 &+\w'\xbf^{T}\w \xbf'^{T} -4{\GGs}^{T}\xbf\w'^{T}\mathbf{X}'\mathbf{X}'^T-4{\GGs}^{T}\X\X^T\w \xbf'^{T}+4 {\GGs}^{T} \X\X^T\GGs \mathbf{X}'\mathbf{X}'^T \big] \\
 &=\tr\big[ 2 \w' \xbf^{T} \w \xbf'^{T}+2{\GGs}^{T}\xbf\xbf'^{T} -4{\GGs}^{T}\xbf\w'^{T}\mathbf{X}'\mathbf{X}'^T-4{\GGs}^{T}\X\X^T\w \xbf'^{T}+4 {\GGs}^{T} \X\X^T\GGs \mathbf{X}'\mathbf{X}'^T \big] \\
 &=2\w^{T}\xbf \w'^{T}\xbf'+ 2 \langle \xbf\xbf'^T-2\xbf\w^{T}\mathbf{X}'\mathbf{X}'^T-2\X\X^T\w\xbf'^{T},\GGs\rangle +4\tr({\GGs}^{T} \X\X^T\GGs \mathbf{X}'\mathbf{X}'^T).
 \end{split}
\end{equation*}

The term $2\w^{T}\xbf \w'^{T}\xbf'$ is constant since it does not depend on the coupling. Also, we can verify that $c_{\mathbf{C},\mathbf{C}'}$ does not depend on $\GGs$ as follows:
\begin{equation*}
\begin{split}
\langle c_{\mathbf{C},\mathbf{C}'}, \GGs\rangle&=\sum_{i}\|\xbf_i-\xbf_j\|_{2}^{4} \w_i \w_j + \sum_{i}\|\xbf'_i-\xbf'_j\|_{2}^{4} \w'_i \w'_j
 \end{split}
\end{equation*}
implying that:
\begin{equation*}
\begin{split}
 &\langle c_{\mathbf{C},\mathbf{C}'}-2\mathbf{C}\GGs \mathbf{C}', \GGs \rangle = Cte-4 \langle \xbf\xbf'^T-2\xbf\w^{T}\mathbf{X}'\mathbf{X}'^T-2\X^{T}\X\w\xbf'^{T}, \GGs \rangle -8 \tr({\GGs}^{T} \X\X^T\GGs \mathbf{X}'\mathbf{X}'^T).
 \end{split}
\end{equation*}
We can rewrite this equation as stated in the proposition using the $\vec$ operator. 

Using a standard QP form $\mathbf{c}^T \xbf +\frac{1}{2}\xbf \mathbf{Q}' \xbf^{T}$ with $\mathbf{c}=-4\vec(\mathbf{M})$ and $\mathbf{Q}'=-4\X\X^T \otimes_{K} \X'\X'^T$ we see that the Hessian is negative semi-definite as the opposite of a Kronecker product of positive semi-definite matrices $\X\X^T$ and $\X'\X'^T$.
\end{proof}

Using previous propositions we are able to prove the Proposition 3 of the paper.
\begin{prop}
\label{prop:cot_equal_gw_sup}
Let $\mathbf{C} \in \R^{n\times n},\mathbf{C}' \in \R^{n'\times n'}$ be any symmetric matrices, then: 
$$\COOT(\mathbf{C},\mathbf{C}',\w,\w',\w,\w')\leq GW(\mathbf{C},\mathbf{C}',\w,\w').$$
The converse is also true {under the hypothesis of Proposition \ref{concavity_gw_theo}}. In this case, if $(\GGs_{*},\GGv_{*})$ is an optimal solution of
\eqref{eq:co-optimal-transport}, then both $\GGs_{*},\GGv_{*}$ are solutions of
\eqref{eq:gromov_sup}. Conversely, if $\GGs_{*}$ is an optimal solution of
\eqref{eq:gromov_sup}, then $(\GGs_{*},\GGs_{*})$ is an optimal solution for
\eqref{eq:co-optimal-transport} .
\end{prop}

\begin{proof}
The first inequality follows from the fact that any optimal solution of the GW problem is an admissible solution for the \COOT\ problem, hence the inequality is true by suboptimality of this optimal solution.

For the equality part, by following the same calculus as in the proof of Proposition \ref{concavity_gw_theo_sup}, we can verify that:
\begin{equation*}
\begin{split}
\COOT(\mathbf{C},\mathbf{C}',\w,\w',\w,\w')&=\min_{\GGs \in\Pi(\w,\w')} -2\vec(\mathbf{M})^{T}\vec(\GGs)\\
&-2\vec(\mathbf{M})^{T}\vec(\GGv) -8\vec(\GGs)^{T}\mathbf{Q}\vec(\GGv) +Cte,
\end{split}
\end{equation*}
with $\mathbf{M},\mathbf{Q}$ as defined in Proposition \ref{concavity_gw_theo_sup}.

Since $\mathbf{Q}$ is negative semi-definite, we can apply Theorem \ref{equivalence_theo} to prove that both problems are equivalent and lead to the same cost and that every optimal solution of GW is an optimal solution of \COOT\ and vice versa.
\end{proof}

\subsection{Equivalence of DC algorithm and Frank-Wolfe algorithm for GW}

Let us first recall the general algorithm used for solving \COOT\ for arbitrary datasets.
\begin{algorithm}[H]
\caption{\label{alg:bcd_coot}
   BCD for \COOT}
\begin{algorithmic}[1]
        \State \textbf{Input:} maxIt, thd
        \State $\pi^{s}_{(0)}\leftarrow \w\w'^{T},\pi^{v}_{(0)}\leftarrow \v\v'^{T}, k \leftarrow 0$
          \While {$k < $ maxIt {\bf and} $err >$ thd} 
          \State $\GGv_{(k)} \leftarrow OT(\v,\v',L(\X,\X')\otimes \GGs_{(k-1)})$
          \State $\GGs_{(k)} \leftarrow OT(\w,\w',L(\X,\X')\otimes \GGv_{(k-1)})$
          \State $err \leftarrow ||\GGv_{(k-1)} - \GGv_{(k)}||_F$
          \State $k\leftarrow k+1$          
          \EndWhile
        \end{algorithmic}
\end{algorithm}

Using Proposition \ref{prop:cot_equal_gw_sup}, we know that when $\mathbf{X}=\mathbf{C}$, $\mathbf{X}'=\mathbf{C}'$ are squared Euclidean matrices, then there is an optimal solution of the form $(\GG^{*},\GG^{*})$. In this case, we can set $\GGs_{(k)}=\GGv_{(k)}$ during the iterations of Algorithm \ref{alg:bcd_coot} to obtain an optimal solution for both \COOT\ and GW. This reduces to Algorithm \ref{alg:bcd2_sup} that corresponds to a DC algorithm where the quadratic form is replaced by its linear upper bound. 

Below, we prove that this DC algorithm for solving GW problems is equivalent to the Frank-Wolfe (FW) based algorithm presented in \cite{vay2019fgw} and recalled in Algorithm \ref{alg:cg_gw} when $L=|\cdot|^2$ and for squared Euclidean distance matrices $\mathbf{C}', \mathbf{C}''$. 


\begin{algorithm}[H]
    \caption{\label{alg:bcd2_sup}
     DC Algorithm for \COOT\ and GW with squared Euclidean distance matrices}
        \begin{algorithmic}[1]
        \State \textbf{Input:} maxIt, thd
            \State $\pi^{s}_{(0)}\leftarrow \w \w'^{T}$
            \While {$k < $ maxIt {\bf and} $err >$ thd} 
            \State $\GGs_{(k)} \leftarrow OT(\w,\w',L(\mathbf{C},\mathbf{C}')\otimes \GGs_{(k-1)})$
           \State $err \leftarrow ||\GGs_{(k-1)} - \GGs_{(k)}||_F$
          \State $k\leftarrow k+1$          
          \EndWhile
        \end{algorithmic}
\end{algorithm}

\begin{algorithm}[H]
    \caption{\label{alg:cg_gw}
     FW Algorithm for GW \cite{vay2019fgw}}
    \begin{algorithmic}[1]
    \State \textbf{Input:} maxIt, thd
        \State $\pi^{(0)}\leftarrow \w\w'^\top$
        \While {$k < $ maxIt {\bf and} $err >$ thd} 
        \State $\mathbf{G}\leftarrow$ Gradient from Eq. \eqref{eq:gromov_sup} \emph{w.r.t.} $\GGs_{(k-1)}$
        \State $\tilde\GGs_{(k)}\leftarrow OT(\w,\w', \mathbf{G})$
        \State $\mathbf{z}_{k}(\tau) \leftarrow \GGs_{(k-1)}+\tau(\tilde\GGs_{(k)}-\GGs_{(k-1)})$ for $\tau\in(0,1)$
        \State $\tau^{(k)}\leftarrow \underset{\tau\in(0,1)}{\text{argmin}} \langle L(\mathbf{C},\mathbf{C}') \otimes \mathbf{z}_{k}(\tau), \mathbf{z}_{k}(\tau) \rangle$
        \State $\GGs_{(k)}\leftarrow (1-\tau^{(k)})\GGs_{(k-1)}+\tau^{(k)}\tilde\GGs_{(k)} $
         \State $err \leftarrow ||\GGs_{(k-1)} - \GGs_{(k)}||_F$
          \State $k\leftarrow k+1$          
        \EndWhile
    \end{algorithmic}
\end{algorithm}

The case when $L=|\cdot|^{2}$ and $\mathbf{C},\mathbf{C}'$ are squared Euclidean distance matrices has interesting implications in practice, since in this case the resulting GW problem is a concave QP (as explained in the paper and shown in Lemma \ref{concavity_gw_theo_sup} of this supplementary). In \cite{Maron:2018}, the authors investigated the solution to QP with \emph{conditionally concave energies} using a FW algorithm and showed that in this case the line-search step of the FW is always $1$. Moreover, as shown in Proposition \ref{concavity_gw_theo_sup}, the GW problem can be written as a concave QP with concave energy and is minimizing \textit{a fortiori} a conditionally concave energy. Consequently, the line-search step of the FW algorithm proposed in \cite{vay2019fgw} and described in Algorithm\ref{alg:cg_gw} always leads to an optimal line-search step of $1$. In this case, the Algorithm.\ref{alg:cg_gw} is equivalent to Algorithm \ref{alg:cg_gw_squared} goven below, since $\tau^{(k)}=1$ for all $k$.
\begin{algorithm}[H]
    \caption{\label{alg:cg_gw_squared}
     FW Algorithm for GW with squared Euclidean distance matrices}
    \begin{algorithmic}[1]
    \State \textbf{Input:} maxIt, thd
        \State $\pi^{(0)}\leftarrow \w\w'^\top$
        \While {$k < $ maxIt {\bf and} $err >$ thd} 
        \State $\mathbf{G}\leftarrow$ Gradient from Eq. \eqref{eq:gromov_sup} \emph{w.r.t.} $\GGs_{(k-1)}$
        \State $\GGs_{(k)}\leftarrow OT(\w,\w', \mathbf{G})$
         \State $err \leftarrow ||\GGs_{(k-1)} - \GGs_{(k)}||_F$
          \State $k\leftarrow k+1$   
        \EndWhile
    \end{algorithmic}
\end{algorithm}

Finally, by noticing that in the step 3 of Algorithm \ref{alg:cg_gw_squared} the gradient of \eqref{eq:gromov_sup} \textit{w.r.t } $\GGs_{(k-1)}$ is $2L(\mathbf{C},\mathbf{C}')\otimes \GGs_{(k-1)}$, which gives the same OT solution as for the OT problem in step 3 of Algorithm \ref{alg:bcd2_sup}, we can conclude that the iterations of both algorithms are equivalent.

\subsection{Relation with Invariant OT}
\label{sec:invot}
The objective of this part is to prove the connections between GW, \COOT\ and $\text{InvOT}$ \cite{alavarez:2019} defined as follows: 
\begin{align*}
    \text{InvOT}_p^L(\X,\X') := \min_{\GG \in \Pi(\w,\w')} \min_{f \in \mathcal{F}_p} \ \langle \mathbf{M}_f, \GG \rangle_F,
    \label{eq:OTalvarez}
\end{align*}
where $(\mathbf{M}_f)_{ij} = L(\x_i,f(\x'_j))$ and $\mathcal{F}_p$ is a space of matrices with bounded Shatten p-norms, \ie, $\mathcal{F}_p = \{\mathbf{P} \in \R^{d\times d}:||\mathbf{P}||_p\leq k_p\}$. 

We prove the following result.
\begin{prop}
\label{prop:cot_equal_gw_equal_invariantot_sup}
Using previous notations, $L = |\cdot|^{2}$, $p=2$, (\textit{i.e} $\mathcal{F}_2 = \{\mathbf{P} \in \R^{d\times d}:||\mathbf{P}||_{F}=\sqrt{d}\}$) and cosine similarities $\mathbf{C}=\X\X^T,\mathbf{C}'=\X'\X'^T$. Suppose that $\X'$ is $\w'$-whitened, \textit{i.e} $\X'^{T}\text{diag}(\w')\X=I$. Then, $\text{InvOT}_2^L(\X, \X')$, $\COOT(\mathbf{C},\mathbf{C}')$ and $GW(\mathbf{C},\mathbf{C}')$ are equivalent, namely any optimal coupling of one of this problem is a solution to others problems.
\end{prop}

In order to prove this proposition, we will need the following proposition:
\begin{prop}
\label{prop:wrintinqapbap}
If $L=|\cdot|^{2}$, we have that 
\begin{equation*}
\begin{split}
&GW(\mathbf{C},\mathbf{C}',\w,\w')=\!\!\!\!\!\!\min_{\GGs \in\Pi(\w,\w')} \!\!\!\!\mathbf{c}^{T} \vec(\GGs)+\frac{1}{2} \vec(\GGs)^{T} \mathbf{Q} \vec(\GGs) \\
&\COOT(\mathbf{C},\mathbf{C}',\w,\w')=\!\!\!\!\!\!\min_{\GGs,\GGv \in\Pi(\w,\w')} \frac{1}{2}\mathbf{c}^{T} \vec(\GGs)+\frac{1}{2}\mathbf{c}^{T}\vec(\GGv)+\frac{1}{2} \vec(\GGs)^{T} \mathbf{Q} \vec(\GGv).
\end{split}
\end{equation*} 
with $\mathbf{Q}=-4\mathbf{C} \otimes \mathbf{C}', \mathbf{c}=\text{vec}(\mathbf{C}\w \one_{n'}^{T}+\one_{n}\w'\mathbf{C}')$.
\end{prop}

\begin{proof}
For GW, we refer the reader to \cite[Equation 6]{vay2019fgw}. For \COOT\, we have:
{\small
\begin{equation*}
\begin{split}
&\COOT(\mathbf{C},\mathbf{C}',\w,\w')=\min_{\GGs \in\Pi(\w,\w'),\GGv\in\Pi(\w,\w')} \langle \mathbf{L}(\mathbf{C},\mathbf{C}') \otimes \GGs, \GGv \rangle \\
&= \min_{\GGs \in\Pi(\w,\w'),\GGv\in\Pi(\w,\w')} \frac{1}{2} \langle \mathbf{L}(\mathbf{C},\mathbf{C}') \otimes \GGs, \GGv \rangle + \frac{1}{2} \langle \mathbf{L}(\mathbf{C},\mathbf{C}') \otimes \GGs, \GGv \rangle \\
&= \min_{\GGs \in\Pi(\w,\w'),\GGv\in\Pi(\w,\w')} \frac{1}{2} \langle \mathbf{L}(\mathbf{C},\mathbf{C}') \otimes \GGs, \GGv \rangle + \frac{1}{2} \langle \mathbf{L}(\mathbf{C},\mathbf{C}') \otimes \GGv, \GGs \rangle \\
&=\min_{\GGs \in\Pi(\w,\w'),\GGv\in\Pi(\w,\w')} \frac{1}{2} \langle \mathbf{C}\w \one_{n'}^{T}+\one_{n}\w'\mathbf{C}',\GGs \rangle + \frac{1}{2} \langle \mathbf{C}\w \one_{n'}^{T}+\one_{n}\w'\mathbf{C}',\GGv \rangle -2\langle \mathbf{C} \GGs \mathbf{C}', \GGv \rangle.
\end{split}
\end{equation*}
}
Last equality gives the desired result.
\end{proof}

\begin{proof}[Proof of Proposition \ref{prop:cot_equal_gw_equal_invariantot_sup}]
Without loss of generality, we suppose that the columns of $\mathbf{C}=\X\X^T,\mathbf{C}'=\X'\X'^T$ are normalized. Then, we know from \cite[Lemma 4.3]{alavarez:2019}, that $GW(\mathbf{C},\mathbf{C}')$ and $\text{InvOT}_2^{||\cdot||_2^2}(\X, \X')$ are equivalent. It suffices to show that $GW(\mathbf{C},\mathbf{C}')$ and $\COOT(\mathbf{C},\mathbf{C}')$ are equivalent.
By virtue of Proposition \ref{prop:wrintinqapbap} the $\mathbf{Q}$ associated with the QP and BP problems of GW and \COOT\ is $\mathbf{Q}=-4\X\X^T\otimes_{K} \X'\X'^T$ which is a negative semi-definite matrix. This allows us to apply Theorem \ref{equivalence_theo} to prove that $GW(\mathbf{C},\mathbf{C}')$ and $\COOT(\mathbf{C},\mathbf{C}')$ are equivalent.
\end{proof}

\subsection{Relation with election isomorphism problem}
\label{sec:election_iso}
This section shows that \COOT\ approach can be used to solve the election isomorphism problem defined in \cite{faliszewski19} as follows: let $E=(C,V)$ and $E' = (C',V')$ be two elections, where $C = \{c_1,\dots,c_m\}$ (resp. $C'$) denotes a set of candidates and $V = (v_1, \dots,v_n)$ (resp. $V'$) denotes a set of voters, where each voter $v_i$ has a preference order, also denoted by $v_i$. The two elections $E=(C,V)$ and $E' = (C',V')$, where $\abs{C} = \abs{C'}$, $V = (v_1,\dots,v_n)$, and $V' = (v_1',\dots,v_n')$, are said to be isomorphic if there exists a bijection $\sigma: C \rightarrow C'$ and a permutation $\nu \in S_n$ such that $\sigma(v_i) = v'_{\nu(i)}$ for all $i \in [n]$. The authors further propose a distance underlying this problem defined as follows:
\begin{align*}
        \text{d-ID}(E,E') = \min_{\nu \in S_n} \min_{\sigma \in \Pi(C,C')} \sum_{i=1}^n d\left(\sigma(v_i),v'_{\nu(i)}\right),
\end{align*}
where $S_n$ denotes the set of all permutations over $\{1, \dots, n\}$,
$\Pi(C,C')$ is a set of bijections and $d$ is an arbitrary distance between
preference orders. The authors of \cite{faliszewski19} compute
$\text{d-ID}(E,E')$ in practice by expressing it as the following Integer Linear
Programming problem over the tensor $\mathbf{P}_{ijkl} = M_{ij}N_{kl}$ where
$\mathbf{M} \in \mathbb{R}^{m\times m}$, $\mathbf{N} \in \mathbb{R}^{n\times n}$
\begin{align}
\min_{\mathbf{P}, \mathbf{N}, \mathbf{M}} &\sum_{i,j,k,l} P_{k,l,i,j} \vert \text{pos}_{v_i}(c_k) - \text{pos}_{v'_j}(c'_l) \vert \notag\\
\text{s.t.} \quad & (\mathbf{N}\bm{1}_n)_k = 1,\ \forall k, (\mathbf{N}^\top\bm{1}_n)_l = 1,\ \forall l \label{ref:margelect}\\
							 & (\mathbf{M}\bm{1}_m)_i = 1,\ \forall i, (\mathbf{M}^\top\bm{1}_m)_j = 1,\ \forall j\notag\\
							 & P \leq N_{k,l}, P_{i,j,k,l} \leq M_{i,j}, \ \forall i,j,k,l\notag\\
							 & \sum_{i,k} P_{i,j,k,l} = 1, \ \forall j,l
\label{ref:electionot}
\end{align}
where $\text{pos}_{v_i}(c_k)$ denotes the position of candidate $c_k$ in the
preference order of voter $v_i$. Let us now define two matrices $\X$ and $\X'$
such that $\X_{i,k} = \text{pos}_{v_i}(c_k)$ and $\X'_{j,l} =
\text{pos}_{v'_j}(c'_l)$ and denote by $\GGs_{*},\GGv_{*}$ a minimizer of
$\COOT(\X,\X', \bm{1}_n/n, \bm{1}_n/n, \bm{1}_m/m, \bm{1}_m/m)$ with $L=|\cdot|$ and by
$\mathbf{N}^*, \mathbf{M}^*$ the minimizers of problem \eqref{ref:margelect},
respectively. 

As shown in the main paper, there exists an optimal solution for $\COOT(\X,\X')$ given by permutation matrices as solutions of the Monge-Kantorovich problems for uniform distributions supported on the same number of elements. Then, one may show that the solution of the two problems coincide modulo a multiplicative factor, \ie, $\GGs_{*} = \frac{1}{n} \mathbf{N}^*$ and  $\GGv_{*} = \frac{1}{m} \mathbf{M}^*$ are optimal since $\abs{C} = \abs{C'}$ and $\abs{V} = \abs{V'}$. For $\GGs_{*}$ (the same reasoning holds for $\GGv_{*}$ as well), we have that
\[
  (\GGs_{*})_{ij} = \left\{\begin{array}{l}
    \frac{1}{n}, \quad j = \nu^*_i \\
    0, \quad \text{otherwise}.
  \end{array}\right.
\]
where $\nu^*_i$ is a permutation of voters in the two sets. The only difference
between the two solutions $\GGs_{*}$ and $\mathbf{N}^*$ thus stems from marginal
constraints \eqref{ref:margelect}. To conclude, we note that \COOT\ is a more
general approach as it is applicable for general loss functions $L$, contrary to
the Spearman distance used in \cite{faliszewski19}, and generalizes to the cases
where $n\neq n'$ and $m\neq m'$.

\section{Additional experimental results}
\label{sec:expes_supp}
\subsection{Complementary results for the HDA experiment}
\label{sec:hda_supp}
Here, we present the results for the heterogeneous domain adaptation experiment not included in the main paper due to the lack of space. Table~\ref{tab:D_2_G} follows the same experimental protocol as in the paper but 
shows the two cases where $n_t=1$ and $n_t=5$. Table~\ref{tab:G_2_D_sup} and Table~\ref{tab:G_2_D_unsup} contain the results for the adaptation from GoogleNet to Decaf features,
in a semi-supervised and unsupervised scenarios, respectively Overall, the results are coherent with those from the main paper: in both settings, when $n_t=5$, one can see that the  performance differences between {SGW} and {COOT} is rather significant.
\begin{table}
	\begin{center}	
	\resizebox{\columnwidth}{!}{
	\begin{tabular}{ccccccc}
		\toprule
			\multicolumn{7}{c}{Decaf  $\rightarrow$ GoogleNet }\\
			\midrule
			{Domains} & {Baseline} & {CCA} & {KCCA} & {EGW} & {SGW} & {COOT}\\
						\midrule
			\multicolumn{7}{c}{$n_t=1$}\\
			\midrule
			C$\rightarrow$W & $30.47$$\pm 6.90$ & $13.37$$\pm 7.23$ & $29.21$$\pm 13.14$ & $10.21$$\pm 1.31$ & $\underline{ 66.95}$$\pm 7.61$ & $\bf 77.74$$\pm 4.80$\\
			W$\rightarrow$C & $26.53$$\pm 7.75$ & $16.26$$\pm 5.18$ & $40.68$$\pm 12.02$ & $10.11$$\pm 0.84$ & $\underline{ 80.16}$$\pm 4.78$ & $\bf 87.89$$\pm 2.65$\\
			W$\rightarrow$W & $30.63$$\pm 7.78$ & $13.42$$\pm 1.38$ & $36.74$$\pm 8.38$ & $8.68$$\pm 2.36$ & $\underline{ 78.32}$$\pm 5.86$ & $\bf 89.11$$\pm 2.78$\\
			W$\rightarrow$A & $30.21$$\pm 7.51$ & $12.47$$\pm 2.99$ & $39.11$$\pm 6.85$ & $9.42$$\pm 2.90$ & $\underline{ 80.00}$$\pm 3.24$ & $\bf 89.05$$\pm 2.84$\\
			A$\rightarrow$C & $41.89$$\pm 6.59$ & $12.79$$\pm 2.95$ & $28.84$$\pm 6.24$ & $9.89$$\pm 1.17$ & $\underline{ 72.00}$$\pm 8.91$ & $\bf 84.21$$\pm 3.92$\\
			A$\rightarrow$W & $39.84$$\pm 4.27$ & $19.95$$\pm 23.40$ & $38.16$$\pm 19.30$ & $12.32$$\pm 1.56$ & $\underline{ 75.84}$$\pm 7.37$ & $\bf 89.42$$\pm 4.24$\\
			A$\rightarrow$A & $42.68$$\pm 8.36$ & $15.21$$\pm 7.36$ & $38.26$$\pm 16.99$ & $13.63$$\pm 2.93$ & $\underline{ 75.53}$$\pm 6.25$ & $\bf 91.84$$\pm 2.48$\\
			C$\rightarrow$C & $28.58$$\pm 7.40$ & $18.37$$\pm 17.81$ & $35.11$$\pm 17.96$ & $11.05$$\pm 1.63$ & $\underline{ 61.21}$$\pm 8.43$ & $\bf 78.11$$\pm 5.77$\\
			C$\rightarrow$A & $31.63$$\pm 4.25$ & $15.11$$\pm 5.10$ & $33.84$$\pm 9.10$ & $11.84$$\pm 1.67$ & $\underline{ 66.26}$$\pm 7.95$ & $\bf 82.11$$\pm 2.58$\\
						\midrule
			\bf Mean & $33.61$$\pm 5.77$ & $15.22$$\pm 2.44$ & $35.55$$\pm 3.98$ & $10.80$$\pm 1.47$ & $\underline{ 72.92}$$\pm 6.37$ & $\bf 85.50$$\pm 4.89$\\
			\midrule
			\multicolumn{7}{c}{$n_t=5$}\\
			\midrule
			C$\rightarrow$W & $74.27$$\pm 5.53$ & $14.53$$\pm 7.37$ & $73.27$$\pm 4.99$ & $11.40$$\pm 1.13$ & $\underline{ 84.00}$$\pm 3.99$ & $\bf 85.53$$\pm 2.67$\\
			W$\rightarrow$C & $90.27$$\pm 2.67$ & $21.13$$\pm 6.85$ & $85.00$$\pm 3.44$ & $10.60$$\pm 1.05$ & $\bf 95.20$$\pm 2.84$ & $\underline{ 94.53}$$\pm 1.83$\\
			W$\rightarrow$W & $90.93$$\pm 2.50$ & $15.80$$\pm 3.27$ & $90.67$$\pm 2.95$ & $9.80$$\pm 2.60$ & $\bf 95.40$$\pm 2.47$ & $\underline{ 94.93}$$\pm 2.70$\\
			W$\rightarrow$A & $90.47$$\pm 2.92$ & $16.67$$\pm 4.85$ & $87.93$$\pm 2.47$ & $9.80$$\pm 2.68$ & $\underline{ 95.40}$$\pm 1.53$ & $\bf 95.80$$\pm 2.15$\\
			A$\rightarrow$C & $\underline{ 88.33}$$\pm 2.33$ & $15.73$$\pm 4.64$ & $83.13$$\pm 2.84$ & $10.40$$\pm 1.89$ & $84.47$$\pm 5.81$ & $\bf 91.47$$\pm 1.45$\\
			A$\rightarrow$W & $\underline{ 88.40}$$\pm 3.17$ & $13.60$$\pm 6.25$ & $87.27$$\pm 2.82$ & $11.87$$\pm 2.40$ & $87.87$$\pm 4.66$ & $\bf 93.00$$\pm 1.96$\\
			A$\rightarrow$A & $86.20$$\pm 3.08$ & $14.07$$\pm 2.93$ & $87.00$$\pm 3.48$ & $14.07$$\pm 1.65$ & $\underline{ 89.80}$$\pm 2.58$ & $\bf 92.20$$\pm 1.69$\\
			C$\rightarrow$C & $75.93$$\pm 4.83$ & $13.13$$\pm 2.98$ & $70.47$$\pm 3.45$ & $11.13$$\pm 1.52$ & $\bf 85.73$$\pm 3.54$ & $\underline{ 84.60}$$\pm 2.32$\\
			C$\rightarrow$A & $73.47$$\pm 3.62$ & $15.47$$\pm 6.50$ & $74.13$$\pm 5.42$ & $11.20$$\pm 2.47$ & $\underline{ 85.07}$$\pm 3.26$ & $\bf 87.20$$\pm 1.78$\\
			\midrule
			\bf Mean & $84.25$$\pm 7.01$ & $15.57$$\pm 2.25$ & $82.10$$\pm 7.03$ & $11.14$$\pm 1.23$ & $\underline{ 89.21}$$\pm 4.64$ & $\bf 91.03$$\pm 3.97$\\
			\bottomrule
			
		\end{tabular}
		}
	\end{center}
	\caption{\label{tab:D_2_G}{\bf Semi-supervised Heterogeneous Domain Adaptation} results for adaptation from Decaf  to GoogleNet  representations with different values of $n_t$. Note that the case $n_t$ is provided in the main paper.}
\end{table}

\begin{table}
	\begin{center}
	\resizebox{.9\columnwidth}{!}{
	\begin{tabular}{ccccccc}
		\toprule
			\multicolumn{7}{c}{GoogleNet  $\rightarrow$ Decaf }\\
			\midrule
			{Domains} & {Baseline} & {CCA} & {KCCA} & {EGW} & {SGW} & {COOT}\\
						\midrule
			\multicolumn{7}{c}{$n_t=1$}\\
			\midrule
			C$\rightarrow$A & $31.16$$\pm 6.87$ & $12.16$$\pm 2.78$ & $33.32$$\pm 2.47$ & $7.00$$\pm 2.11$ & $\underline{ 77.16}$$\pm 8.00$ & $\bf 83.26$$\pm 5.00$\\
			C$\rightarrow$C & $30.42$$\pm 3.73$ & $13.74$$\pm 5.29$ & $32.58$$\pm 9.98$ & $12.47$$\pm 2.81$ & $\underline{ 76.63}$$\pm 8.31$ & $\bf 86.21$$\pm 3.26$\\
			W$\rightarrow$A & $37.68$$\pm 4.04$ & $15.79$$\pm 3.71$ & $34.58$$\pm 5.71$ & $14.32$$\pm 1.77$ & $\underline{ 86.68}$$\pm 1.90$ & $\bf 89.95$$\pm 3.43$\\
			A$\rightarrow$C & $35.95$$\pm 3.89$ & $15.32$$\pm 8.18$ & $40.16$$\pm 17.54$ & $13.21$$\pm 3.49$ & $\underline{ 87.89}$$\pm 4.03$ & $\bf 90.68$$\pm 7.54$\\
			A$\rightarrow$A & $36.89$$\pm 4.73$ & $13.84$$\pm 2.47$ & $34.84$$\pm 10.44$ & $13.16$$\pm 1.56$ & $\underline{ 89.79}$$\pm 3.93$ & $\bf 94.68$$\pm 2.21$\\
			W$\rightarrow$W & $32.05$$\pm 4.63$ & $19.89$$\pm 11.82$ & $36.26$$\pm 21.98$ & $10.00$$\pm 2.59$ & $\underline{ 84.21}$$\pm 4.55$ & $\bf 90.42$$\pm 2.66$\\
			W$\rightarrow$C & $32.68$$\pm 5.56$ & $21.53$$\pm 21.01$ & $33.79$$\pm 22.72$ & $11.47$$\pm 3.03$ & $\underline{ 86.26}$$\pm 3.41$ & $\bf 89.53$$\pm 1.92$\\
			A$\rightarrow$W & $33.84$$\pm 4.75$ & $16.00$$\pm 7.74$ & $39.32$$\pm 18.94$ & $11.00$$\pm 4.01$ & $\underline{ 87.21}$$\pm 3.67$ & $\bf 91.53$$\pm 5.85$\\
			C$\rightarrow$W & $32.32$$\pm 7.76$ & $15.58$$\pm 7.72$ & $34.05$$\pm 15.96$ & $12.89$$\pm 2.52$ & $\underline{ 81.84}$$\pm 3.51$ & $\bf 84.84$$\pm 5.71$\\
			\midrule
			\bf Mean & $33.67$$\pm 2.45$ & $15.98$$\pm 2.81$ & $35.43$$\pm 2.50$ & $11.73$$\pm 2.08$ & $\underline{ 84.19}$$\pm 4.43$ & $\bf 89.01$$\pm 3.38$\\
			\midrule
			\multicolumn{7}{c}{$n_t=3$}\\
			\midrule
			C$\rightarrow$A & $76.35$$\pm 4.15$ & $17.47$$\pm 3.45$ & $73.94$$\pm 4.53$ & $7.41$$\pm 2.27$ & $\underline{ 88.24}$$\pm 2.23$ & $\bf 89.88$$\pm 0.94$\\
			C$\rightarrow$C & $78.94$$\pm 3.61$ & $18.18$$\pm 3.44$ & $69.94$$\pm 3.51$ & $14.18$$\pm 3.16$ & $\underline{ 89.71}$$\pm 2.25$ & $\bf 91.06$$\pm 1.91$\\
			W$\rightarrow$A & $85.41$$\pm 3.25$ & $19.29$$\pm 3.10$ & $80.59$$\pm 3.82$ & $14.24$$\pm 2.72$ & $\underline{ 94.76}$$\pm 1.45$ & $\bf 95.29$$\pm 2.35$\\
			A$\rightarrow$C & $89.53$$\pm 4.05$ & $23.18$$\pm 7.17$ & $80.59$$\pm 6.30$ & $13.88$$\pm 2.69$ & $\underline{ 93.76}$$\pm 2.72$ & $\bf 94.76$$\pm 1.83$\\
			A$\rightarrow$A & $89.76$$\pm 1.92$ & $17.00$$\pm 3.11$ & $83.71$$\pm 3.30$ & $14.41$$\pm 2.28$ & $\underline{ 93.29}$$\pm 2.09$ & $\bf 95.53$$\pm 1.45$\\
			W$\rightarrow$W & $86.65$$\pm 5.07$ & $21.88$$\pm 4.78$ & $84.65$$\pm 3.67$ & $9.94$$\pm 2.37$ & $\bf 94.88$$\pm 1.79$ & $\underline{ 94.53}$$\pm 1.66$\\
			W$\rightarrow$C & $88.94$$\pm 5.02$ & $22.59$$\pm 9.23$ & $80.06$$\pm 5.65$ & $13.65$$\pm 3.15$ & $\bf 96.18$$\pm 1.15$ & $\underline{ 95.29}$$\pm 2.91$\\
			A$\rightarrow$W & $90.29$$\pm 1.35$ & $22.35$$\pm 7.00$ & $87.88$$\pm 2.53$ & $13.88$$\pm 3.60$ & $\underline{ 94.53}$$\pm 1.54$ & $\bf 95.35$$\pm 1.59$\\
			C$\rightarrow$W & $78.59$$\pm 3.44$ & $22.53$$\pm 13.42$ & $80.12$$\pm 2.95$ & $11.59$$\pm 3.25$ & $\underline{ 89.29}$$\pm 1.86$ & $\bf 89.59$$\pm 2.22$\\
			\midrule
			\bf Mean & $84.94$$\pm 5.19$ & $20.50$$\pm 2.34$ & $80.16$$\pm 5.12$ & $12.58$$\pm 2.31$ & $\underline{ 92.74}$$\pm 2.72$ & $\bf 93.48$$\pm 2.38$\\
			\midrule
			\multicolumn{7}{c}{$n_t=5$}\\
			\midrule
			C$\rightarrow$A & $84.20$$\pm 2.65$ & $18.60$$\pm 3.75$ & $84.33$$\pm 2.33$ & $6.40$$\pm 1.27$ & $\bf 92.13$$\pm 2.61$ & $\underline{ 91.93}$$\pm 2.05$\\
			C$\rightarrow$C & $85.33$$\pm 2.76$ & $21.80$$\pm 5.91$ & $78.60$$\pm 2.74$ & $13.47$$\pm 2.00$ & $\underline{ 91.33}$$\pm 2.48$ & $\bf 92.27$$\pm 2.67$\\
			W$\rightarrow$A & $95.13$$\pm 2.29$ & $31.00$$\pm 9.67$ & $91.93$$\pm 2.82$ & $14.67$$\pm 1.40$ & $\underline{ 96.13}$$\pm 2.04$ & $\bf 96.40$$\pm 1.84$\\
			A$\rightarrow$C & $91.67$$\pm 2.60$ & $21.80$$\pm 4.35$ & $85.33$$\pm 3.27$ & $13.40$$\pm 3.63$ & $\bf 95.47$$\pm 1.51$ & $\underline{ 94.87}$$\pm 1.27$\\
			A$\rightarrow$A & $93.20$$\pm 1.57$ & $23.33$$\pm 4.66$ & $89.67$$\pm 1.98$ & $13.27$$\pm 2.10$ & $\bf 95.33$$\pm 1.07$ & $\underline{ 95.00}$$\pm 1.37$\\
			W$\rightarrow$W & $95.00$$\pm 2.33$ & $23.80$$\pm 5.48$ & $92.13$$\pm 1.78$ & $11.20$$\pm 2.58$ & $\underline{ 96.47}$$\pm 1.93$ & $\bf 96.67$$\pm 1.37$\\
			W$\rightarrow$C & $95.67$$\pm 1.50$ & $28.27$$\pm 9.71$ & $87.67$$\pm 3.79$ & $14.27$$\pm 3.19$ & $\bf 97.67$$\pm 1.31$ & $\underline{ 96.93}$$\pm 2.25$\\
			A$\rightarrow$W & $92.13$$\pm 2.36$ & $22.67$$\pm 3.94$ & $89.20$$\pm 3.14$ & $11.67$$\pm 2.50$ & $\underline{ 93.60}$$\pm 1.40$ & $\bf 94.27$$\pm 2.11$\\
			C$\rightarrow$W & $84.00$$\pm 3.45$ & $20.40$$\pm 4.31$ & $82.53$$\pm 3.56$ & $11.07$$\pm 3.70$ & $\underline{ 90.20}$$\pm 2.23$ & $\bf 92.40$$\pm 1.69$\\
			\midrule
			\bf Mean & $90.70$$\pm 4.57$ & $23.52$$\pm 3.64$ & $86.82$$\pm 4.26$ & $12.16$$\pm 2.37$ & $\underline{ 94.26}$$\pm 2.42$ & $\bf 94.53$$\pm 1.85$\\
			\bottomrule
				
		\end{tabular}
		}
	\end{center}
	\caption{\label{tab:G_2_D_sup}{\bf Semi-supervised Heterogeneous Domain Adaptation} results for adaptation  from GoogleNet  to Decaf  representations with different values of  $n_t$.}
\end{table}

\begin{table}
	\begin{center}
	\resizebox{0.65\columnwidth}{!}{
		\begin{tabular}{ccccc}
		\toprule
			\multicolumn{5}{c}{GoogleNet  $\rightarrow$ Decaf }\\
			\midrule
			{Domains} & {CCA} & {KCCA} & {EGW} & {COOT}\\
			\midrule
			C$\rightarrow$A & $11.30$$\pm 4.04$ & $\underline{ 14.60}$$\pm 8.12$ & $8.20$$\pm 2.69$ & $\bf 25.10$$\pm 11.52$\\
			C$\rightarrow$C & $13.35$$\pm 4.32$ & $\underline{ 17.75}$$\pm 10.16$ & $11.90$$\pm 2.99$ & $\bf 37.20$$\pm 14.07$\\
			W$\rightarrow$A & $14.55$$\pm 10.68$ & $\underline{ 25.05}$$\pm 24.73$ & $14.55$$\pm 2.05$ & $\bf 39.75$$\pm 17.29$\\
			A$\rightarrow$C & $13.80$$\pm 6.51$ & $\underline{ 20.70}$$\pm 17.94$ & $16.00$$\pm 2.44$ & $\bf 30.25$$\pm 18.71$\\
			A$\rightarrow$A & $16.90$$\pm 10.45$ & $\underline{ 28.95}$$\pm 30.62$ & $12.70$$\pm 1.79$ & $\bf 41.65$$\pm 16.66$\\
			W$\rightarrow$W & $14.50$$\pm 6.72$ & $\underline{ 24.05}$$\pm 19.35$ & $9.55$$\pm 1.77$ & $\bf 36.85$$\pm 9.20$\\
			W$\rightarrow$C & $13.15$$\pm 4.98$ & $\underline{ 14.80}$$\pm 8.79$ & $11.40$$\pm 2.65$ & $\bf 30.95$$\pm 17.18$\\
			A$\rightarrow$W & $10.85$$\pm 4.62$ & $\underline{ 14.40}$$\pm 12.36$ & $12.70$$\pm 2.99$ & $\bf 40.85$$\pm 16.21$\\
			C$\rightarrow$W & $18.25$$\pm 14.02$ & $\underline{ 25.90}$$\pm 25.40$ & $11.30$$\pm 3.87$ & $\bf 34.05$$\pm 13.82$\\
			\bf Mean & $14.07$$\pm 2.25$ & $\underline{ 20.69}$$\pm 5.22$ & $12.03$$\pm 2.23$ & $\bf 35.18$$\pm 5.24$\\
		\bottomrule
		\end{tabular}
		}
	\end{center}
		\caption{	\label{tab:G_2_D_unsup} {\bf Unsupervised Heterogeneous Domain Adaptation} results for adaptation from GoogleNet to Decaf representations.}
\end{table}

\subsection{Complementary information for the co-clustering experiment}
\label{sec:co-clust_supp}
Table \ref{tab:data_description} below summarizes the characteristics of the simulated data sets used in our experiment.
\begin{table}[h]
\centering
\resizebox{0.7\textwidth}{!}{
\begin{tabular}{lccccl}
\hline
Data set&$n\times d$&$g\times m$&Overlapping&Proportions\\
\hline
D1&$600\times300$&$3\times 3$&[+]&Equal\\
D2&$600\times300$&$3\times 3$&[+]&Unequal\\
D3&$300\times200$&$2\times 4$&[++]&Equal\\
D4&$300\times300$&$5\times4$&[++]&Unequal\\
\hline
\end{tabular}
}
\vspace{2mm}
\caption{\label{tab:data_description}Size ($n\times d$), number of co-clusters ($g\times m$), degree of overlapping ([+] for well-separated and [++] for ill-separated co-clusters) and the proportions of co-clusters for simulated data sets.}
\end{table}

\section{Initialization's impact}

We conducted a study regarding the convergence properties of COOT in the co-clustering application when the $\pi_s,\pi_v$ and $X_c$ are initialized randomly over $100$ trials. This leads to a certain variance in the obtained value of the COOT distance as expected when solving a non-convex problem. The obtained CCEs remain largely in line with the obtained results even for different random initializations.

\renewcommand{\arraystretch}{1.3}
\setlength{\intextsep}{5pt}%
\setlength{\columnsep}{5pt}%
\begin{table}[h]
\begin{center}
\resizebox{0.9\textwidth}{!}{
\begin{tabular}{lcccc}
\hline
\multirow{2}{*}{Data set} &\multicolumn{4}{c}{Characteristics}\\
\cline{2-5}
  & Runtime(s) & BCD \#iter. (COOT+$\X_c$) & BCD \#iter. (COOT) & COOT value\\
\hline
D1 & 4.72$\pm$6 & 21.5$\pm$24.57 & 3.16$\pm$0.37 &0.46$\pm$0.25\\
D2 & 0.64$\pm$0.81 & 9.77$\pm$11.53 & 3.4$\pm$0.58 &1.35$\pm$0.16\\
D3 & 0.95$\pm$1.55 & 8.47$\pm$11.11 & 3.01$\pm$0.1 &2.52$\pm$0.24\\
D4 & 6.27$\pm$5.13 & 33.15$\pm$23.75 & 4.21$\pm$0.41 & 0.06$\pm$0.005\\
\hline
\end{tabular}
}
\end{center}
\caption{ Mean ($\pm$ standard-deviation) of different runtime characteristics of COOT.
}
\end{table}

\bibliography{refs.bib}
	
\bibliographystyle{unsrt}

\end{document}